\documentclass[conference]{IEEEtran}
\IEEEoverridecommandlockouts

\usepackage{hyperref} 
\usepackage{graphicx}
\usepackage{booktabs}
\usepackage{amsmath}
\usepackage{amssymb}
\usepackage{amsthm}
\usepackage{bbm}
\usepackage{amsfonts}
\usepackage{epsfig}
\usepackage{subfigure}
\usepackage{url}
\usepackage{diagbox}
\usepackage{balance}
\usepackage[noend]{algorithmic}
\usepackage{algorithm}
\usepackage{multirow}
\usepackage{xspace}
\usepackage{mdframed}
\usepackage[dvipsnames]{xcolor}

\usepackage{pifont}
\newcommand{\cmark}{\ding{51}}%
\newcommand{\xmark}{\ding{55}}%

\newcommand{\be}{\begin{enumerate}}
	\newcommand{\ee}{\end{enumerate}}
\newcommand{\beqn}{\begin{eqnarray*}}
	\newcommand{\eeqn}{\end{eqnarray*}}

\newcommand{\eetitle}[1]{\vspace{0.8ex}\noindent{\underline{\em #1}}}
\newcommand{\etitle}[1]{\vspace{0.5ex}\noindent{{\em #1}}}

\newcommand{\ie}{\emph{i.e.,}\xspace}
\newcommand{\eg}{\emph{e.g.,}\xspace}
\newcommand{\wrt}{\emph{w.r.t.}\xspace}

\newcommand{\eop}{\hspace*{\fill}\mbox{$\Box$}}

 \newcounter{example}
 \renewcommand{\theexample}{\arabic{example}}
 \newenvironment{example}{
       \vspace{1.5ex}
       \refstepcounter{example}
       {\noindent\bf Example \theexample:}}{
         \eop
         \vspace{1.5ex}
         }

\newcommand{\nthesection}{\arabic{section}}

 \newcounter{lemma}
 \renewcommand{\thelemma}{\arabic{lemma}}
 \newenvironment{lemma}{\begin{em}
 		\refstepcounter{lemma}
 		{
 		\vspace{1ex}
 		\noindent\bf Lemma \thelemma:}}{
 	\end{em}\eop
 	\vspace{1ex}
 	}

 \newcounter{axiom}
 \renewcommand{\theaxiom}{\arabic{axiom}}

 \newcounter{theorem}
 \renewcommand{\thetheorem}{\arabic{theorem}}
 \newenvironment{theorem}{\begin{em}
         \refstepcounter{theorem}
         {
         \vspace{1.5ex}
         \noindent\bf  Theorem  \thetheorem:}}{
         \end{em}\eop
         \vspace{1.5ex}
         }

\newenvironment{ctheorem}[1]{\begin{em}
        \refstepcounter{theorem}
        {\vspace{1ex}{\noindent \bf Theorem  \thetheorem~[#1]}: }}{
        \end{em}\eop\vspace{1.5ex}}

\newcounter{cor}
\renewcommand{\thecor}{\arabic{cor}}

\newcounter{prop}
\renewcommand{\theprop}{\arabic{theorem}}

 \newcounter{definition}[section]
 \renewcommand{\thedefinition}{\nthesection.\arabic{definition}}

\newcounter{alg}[section]
\renewcommand{\thealg}{\nthesection.\arabic{alg}}

\newcounter{arule}
\renewcommand{\thearule}{\arabic{arule}}

\newenvironment{proofS}{
	\vspace{1ex}
	{\noindent\bf Proof sketch:\ }}{\eop
	\vspace{1ex}
	}

 \newcommand{\tbf}{\textbf{\textcolor{red}{TBF}}\xspace}

\newcommand{\squishlist}{
 \begin{list}{$\bullet$}
  {  \setlength{\itemsep}{0pt}
     \setlength{\parsep}{3pt}
     \setlength{\topsep}{3pt}
     \setlength{\partopsep}{0pt}
     \setlength{\leftmargin}{2em}
     \setlength{\labelwidth}{1.5em}
     \setlength{\labelsep}{0.5em}
} }
\newcommand{\squishlisttight}{
 \begin{list}{$\bullet$}
  { \setlength{\itemsep}{0pt}
    \setlength{\parsep}{0pt}
    \setlength{\topsep}{0pt}
    \setlength{\partopsep}{0pt}
    \setlength{\leftmargin}{2em}
    \setlength{\labelwidth}{1.5em}
    \setlength{\labelsep}{0.5em}
} }

\newcommand{\squishdesc}{
 \begin{list}{}
  {  \setlength{\itemsep}{0pt}
     \setlength{\parsep}{3pt}
     \setlength{\topsep}{3pt}
     \setlength{\partopsep}{0pt}
     \setlength{\leftmargin}{1em}
     \setlength{\labelwidth}{1.5em}
     \setlength{\labelsep}{0.5em}
} }

\newcommand{\squishend}{
  \end{list}
}

\renewcommand{\And}{\mbox{\bf and}\ }

\newcommand{\eat}[1]{}

\newcommand{\NP}{\kw{NP}}

\newcommand{\kw}[1]{{\ensuremath {\mathsf{#1}}}\xspace}

\newcommand{\stitle}[1]{\vspace{1.5ex}\noindent{\bf #1}}

\newcommand{\sstab}{\rule{0pt}{8pt}\\[-2ex]}
\newcounter{ccc}

\DeclareMathOperator*{\argmax}{arg\,max}

\usepackage[show]{boxnotes}
\usepackage[normalem]{ulem}
\newcommand\redout{\bgroup\markoverwith
{\textcolor{red}{\rule[.5ex]{2pt}{2pt}}}\ULon}

\newcommand{\wu}[1]{{\color{orange}[com-Wu:~{#1}]}}

\newcommand{\dq}[1]{{\color{blue}[{#1}]}}
\newcommand{\mengying}[1]{{\color{cyan}[MY:~{#1}]}}

\newcommand{\rcexp}{\kw{RCExplainer}}

\newcommand{\comwu}[1]{{\color{red}[com-Wu:~{#1}]}}
\newcommand{\warn}[1]{{\color{red}[{#1}]}}

\newcommand{\V}{{\mathcal V}}

\newcommand{\M}{{\mathcal M}}
\newcommand{\C}{{\mathcal C}}
\newcommand{\gnn}{\kw{GNN}}
\newcommand{\gnns}{\kw{GNNs}}
\newcommand{\gat}{\kw{GAT}}
\newcommand{\gats}{\kw{GATs}}
\newcommand{\gcn}{\kw{GCN}}
\newcommand{\gcns}{\kw{GCNs}}
\newcommand{\ppnps}{\kw{APPNPs}}
\newcommand{\ppnp}{\kw{APPNP}}
\newcommand{\gsage}{\kw{GraphSage}}

\newcommand{\PTIME}{\kw{PTIME}}

\newcommand{\gvex}{\kw{GVEX}}

\renewenvironment{proof}{
        \vspace{1ex}
        {\noindent\bf Proof:}}{\vspace{1ex}}

\renewcommand{\st}[1]{}

\newcommand{\cw}{\kw{CW}}
\newcommand{\rcw}{\kw{RCW}}
\newcommand{\rcws}{\kw{RCWs}}
\newcommand{\rgexp}{\kw{RoboGExp}}
\newcommand{\prgexp}{\kw{paraRoboGExp}}

\newcommand{\ppi}{\kw{PPI}}
\newcommand{\house}{\kw{BAHouse}}
\newcommand{\citeseer}{\kw{CiteSeer}}
\newcommand{\reddit}{\kw{Reddit}}
\newcommand{\cf}{\kw{CF\textsuperscript{2}}}
\newcommand{\gcf}{\kw{GCFExplainer}}
\newcommand{\cfexplainer}{\kw{CF}-\kw{GNNExp}}

\newcommand{\verifyw}{\kw{verifyW}}
\newcommand{\verifycw}{\kw{verifyCW}}
\newcommand{\verifyrcw}{\kw{verifyRCW}}
\newcommand{\verifyrcwp}{\kw{verifyRCW}-\kw{APPNP}}
\newcommand{\lbp}{\kw{LBP}}

\newcommand{\pri}{\kw{PRI}}

\newcommand{\mrcwgen}{\kw{RoboGExp}}
\newcommand{\prcwgen}{\kw{paraRoboGExp}}
\newcommand{\expand}{\kw{Expand}}

\newcommand{\paraexpand}{\kw{paraExpand}}
\newcommand{\paraverifyrcw}{\kw{paraverifyRCW}}

\newcommand{\ged}{\kw{GED}}
\newcommand{\fplus}{\kw{Fidelity+}}
\newcommand{\fminus}{\kw{Fidelity-}}
\usepackage{cancel}
\usepackage{titlesec}

\eat{
\AtBeginDocument{%
  \providecommand\BibTeX{{%
    \normalfont B\kern-0.5em{\scshape i\kern-0.25em b}\kern-0.8em\TeX}}}

\setcopyright{acmcopyright}
\copyrightyear{2024}
\acmYear{2024}
\acmDOI{XXXXXXX.XXXXXXX}

\acmConference[SIGMOD '24]{ACM International Conference on Management of Data}{June 11--16, 2024}{Santiago, Chile}

\acmPrice{15.00}
\acmISBN{978-1-4503-XXXX-X/18/06}

\settopmatter{printacmref=false}
\setcopyright{none}
\renewcommand\footnotetextcopyrightpermission[1]{}
\pagestyle{plain}

\pagenumbering{⟨arabic⟩}
}

\begin{document}

\title{
Generating Robust Counterfactual Witnesses for 
Graph Neural Networks}

\author{\IEEEauthorblockN{Dazhuo Qiu}
\IEEEauthorblockA{\textit{Aalborg University} \\
Denmark \\
dazhuoq@cs.aau.dk}
\and
\IEEEauthorblockN{Mengying Wang}
\IEEEauthorblockA{\textit{Case Western Reserve University} \\
USA\\
mxw767@case.edu}
\and
\IEEEauthorblockN{Arijit Khan}
\IEEEauthorblockA{\textit{Aalborg University} \\
Denmark \\
arijitk@cs.aau.dk}
\and
\IEEEauthorblockN{Yinghui Wu}
\IEEEauthorblockA{\textit{Case Western Reserve University} \\
USA\\
yxw1650@case.edu}
}


\maketitle

\begin{abstract}
This paper introduces a new class of 
explanation structures, called 
{\em robust counterfactual witnesses} (\rcws), to provide 
robust, both counterfactual and factual 
explanations 
for graph neural networks. 
Given a graph neural network $\M$, 
a robust counterfactual witness 
refers to the fraction of 
a graph $G$ that 
are counterfactual and factual 
explanation of the results of $\M$
over $G$, 
but also remains so 
for any ``disturbed'' $G$ 
by flipping up to $k$ of its node pairs. 
We establish the hardness results, 
from tractable results to 
co-\NP-hardness, for 
verifying and 
generating robust counterfactual witnesses. 
We study such structures for \gnn-based 
node classification, and present efficient algorithms to 
verify and generate \rcws. We also 
provide a parallel algorithm 
to verify and generate \rcws for 
large graphs with scalability 
guarantees. 
We experimentally verify our 
explanation generation process 
for benchmark datasets, and 
showcase their applications. 
\end{abstract}

\eat{


}


\vspace{-1mm}
\section{Introduction}
\label{sec:intro}
Graph neural networks  (\gnns) have exhibited 
promising performances in graph analytical 
tasks such as classification.  
Given a graph $G$ and a set of test 
nodes $V_T$, a \gnn-based node classifier 
aims to assign a correct label to each node $v\in V_T$. \gnn-based classification 
has been applied for applications  
such as social networks and biochemistry~\cite{you2018graph, cho2011friendship,wei2023neural}.

For many \gnns-based tasks such as 
node classification, data analysts or
domain scientists often want 
to understand the 
results of \gnns by inspecting intuitive, explanatory  structures, 
in the form where 
domain knowledge can be 
directly applied~\cite{gvex24}. In 
particular, such explanation 
should indicate ``invariant''
representative structures 
for similar 
graphs  
that fall into the same group, 
\ie be ``robust'' to small 
changes of the graphs, 
and be both ``factual'' (that preserves 
the result of classification) and  
``counterfactual'' 
(which flips the result 
if removed from $G$)~\cite{zhou2021evaluating}.
\eat{
Such structures may be 
inherently hard to maintain 
upon {\em changes}, which 
may be caused by adversarial attacks, 
missing links, or noises in the  
underlying graphs, and 
may not easily generalize to 
similar graphs. 
}
The need for such robust, 
and both factual and counterfactual 
explanation 
is evident for real-world applications such as 
drug design or cyber security. 

\vspace{.5ex}
Consider the following real-world examples. 

\eat{
\begin{table*}[tb!]
    \centering
    \vspace{-2mm}
     \caption{\small Comparison of our \gvex technique with state-of-the-art \gnn explanation methods. Here ``Learning'' denotes whether (node/edge mask) learning is required, ``Task'' means what downstream tasks each method can be applied to (GC/NC: graph/ node classification), ``Target'' indicates the output 
     format of explanations (E/NF: Edge/Node Features), ``Model-Agnostic'' means if the method treats \gnns as a black-box during the explanation stage (i.e., the internals of the \gnn models are not required), ``Label-specific" means if the explanations can be generated for a specific class label; ``Size-bound'' means if the size of explanation is bounded; ``Coverage'' means if the coverage property is involved (\S \ref{sec-view}), ``Configurable'' means if users can configure the method to generate explanations for 
     designated class labels (\S \ref{sec:preliminaries}); ``Queryable'' means if the explanations are 
     directly queryable.} 
    \label{tab:comprehensive-analysis}
    \vspace{-3mm}
    \scriptsize
    \begin{tabular}{c c c c c c c c c c}
            \textbf{Methods} & \textsc{Learning} & \textsc{Task} & \textsc{Target} & \textsc{Model-Agnostic} & \textsc{Label-specific} & \textsc{Size-Bounded} & \textsc{Coverage} & \textsc{Configurable} & \textsc{Queryable} \\ 
            \midrule
		  {\bf SubgraphX}~\cite{yuan2021explainability} & \xmark & GC/NC & Subgraph & \cmark & \xmark & \xmark & \xmark & \xmark & \xmark\\ 
            {\bf GNNExplainer}~\cite{ying2019gnnexplainer} & \cmark & GC/NC & E/NF & \cmark & \xmark & \xmark & \xmark & \xmark & \xmark\\
            {\bf PGExplainer}~\cite{luo2020parameterized} & \cmark & GC/NC & E & \xmark & \xmark & \xmark & \xmark & \xmark & \xmark\\
            {\bf GStarX}~\cite{zhang2022gstarx} & \xmark & GC & Subgraph & \cmark & \xmark & \xmark & \xmark & \xmark & \xmark\\
            {\bf GCFExplainer}~\cite{huang2023global} & \xmark  & GC & Subgraph & \cmark & \cmark & \xmark & \cmark & \xmark & \xmark\\
            \midrule
            {\bf \gvex (Ours)} & \xmark  & GC/NC & \begin{tabular}{@{}c@{}}Graph Views \\ (Pattern+Subgraph)\end{tabular} & \cmark & \cmark & \cmark & \cmark & \cmark & \cmark \\
    \end{tabular}
    \vspace{-4mm}
\end{table*}
}

\vspace{1ex}
\begin{example}
\label{exa-motivation}
\textbf{[Interpreting ``Mutagenics'' with 
Molecular Structures]}. 
In drug discovery, \gnns have been applied to 
detect 
{\em mutagenicity} structures, 
which has an adverse ability 
that causes mutations~\cite{xiong2021graph, jiang2020drug}. 
Consider $G_1$ depicted in Fig. \ref{fig-motivation}, which is a graph representation of a real-world molecular database~\cite{david2020molecular,albalahi2022bond}. The nodes of $G_1$ refer to the atoms of the molecules, and the edges of $G_1$ represent the valence bonds between a pair of atoms. Given a set of ``test nodes'' 
$V_T$ in $G$, a \gnn-based classifier $M$ is trained for the task to correctly assign, for each of the nodes 
in $V_T$, a label: ``\kw{mutagenic}'', if it 
belongs to a chemical compound that is a mutagenicity structure; and ``\kw{nonmutagenic}'', 
otherwise. 

A natural interpretation to a chemist would be to identify a subgraph $G_1$ that contains most of the ``mutagenic'' nodes as a {\em toxicophore} structure, 
that is, a critical fraction of chemical compounds responsible for mutagenicity~\cite{kazius2005derivation}. Consider the example in Figure \ref{fig-motivation}. On the left is a mutagenic molecule due to the bold nitro group. Meanwhile, the remaining structure after removing the nitro group is considered as non-mutagenic. However, considering a similar molecule that only misses two bonds (dotted edges), which is in the middle, the red bold structure is an aldehyde with a high risk of triggering mutagenicity. To ensure that the discovered mutagenic-related chemical compounds are meaningful to a family of molecules, it is required to find a robust explanation structure, such as the bold indicated structure on the right molecule. 
\eat{
\textcolor{blue}{A natural interpretation in the eye of 
a chemical analyst identifies a subgraph 
$G_1$ that contains most of the 
``mutagenic'' nodes as a toxicophore structure  
of a family of 
{\em chemical mutagens}, a critical 
fraction of chemical compounds 
responsible for mutagenicity. 
Ideally, an explanation mechanism 
for \gnn-based classifier should 
be able to report such small ``invariant'' 
structures, such as a Quinoline ring, 
that corresponds to common 
structures (known as ``toxicophores'') 
of a family of chemical mutagens 
including Quinoline, Naphthalene, and Pyridine. 
The latters differ only up to  
a bounded number of valence bonds. 
In other words, Quinoline ring remains ``robust'' to 
explain the results of node classification for 
similar molecules with small bond differences.}
}

\begin{figure}[tb!]
\vspace{-1mm}
\centering
\centerline{\includegraphics[width =0.46\textwidth]{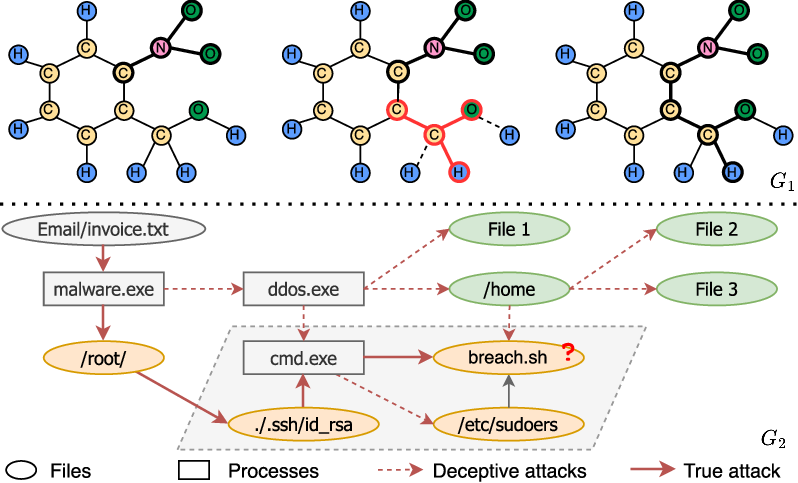}}
\vspace{-2mm}
\caption{\textbf{$G_1$:} \textbf{Left-}Bold nodes and edges indicate a counterfactual. \textbf{Middle-}Red bold nodes and edges indicate a new counterfactual after deleting dotted edges. \textbf{Right-}A counterfactual robust to graph edits. \textbf{$G_2$:} The \textbf{shaded area} is a robust counterfactual, also ``vulnerable zone'' in cyber networks.}
\eat{
Given an input graph, we find a classic factual subgraph and our witness in top-middle and bottom-middle, respectively. However, if we remove $k=2$ edges from the input graph(dotted edges), the counterfactual explanation will change its prediction since the majority of connected nodes reverse from blue to yellow. Meanwhile, for the counterfactual witness, the prediction remains the same. }
\label{fig-motivation}
\vspace{-5mm}
\end{figure}

\vspace{.5ex}
Existing \gnn explanation methods may return different structures as the test nodes vary, 
or carry ``noisy'' structures such as carbon rings or hydrogen atoms that are not 
necessarily corresponding to toxicophores (as verified in 
our case study; see Section~\ref{sec:exp}). 
This may make it difficult for a chemist to discern the potential toxic structure hidden inside normal structures. (e.g., red bold aldehyde in the middle molecule of Figure \ref{fig-motivation}.) 
Moreover, instead of analyzing explanations for each molecule, identifying a common ``invariant'' explanation that remains the same over a family of similar molecules with few bond differences could improve the efficiency of identifying mutagenic-related chemical compounds. A more effective approach is expected to directly compute explanation structures which are, or contain such toxicophores as robust structures, closer to what a chemist is familiar with.
\eat{
\textcolor{blue}{Existing \gnn explanation method may return different structures as the test nodes vary. 
These structures  
may carry ``noisy'' structures such as carbon rings or hydrogen atoms, that are not 
necessary corresponding toxicophores.  
This makes it difficult for a chemist to discern whether the binding of carbon rings or hydrogen atoms on the carbon rings plays a decisive role to decide mutagenicity. 
Moreover, when chemical analysts access such 
explanations, they should benefit from queries that directly express critical toxicophores structures (\eg an aromatic nitro group)~\cite{kazius2005derivation}. A more effective  approach is expected to directly compute explanation structures which are, or contain such toxicophores 
as robust structures closer to what a chemist 
is familiar with.}
}
\eat{
\warn{Give an example -- similarly as 
GVex, shows the need for queryable structure, 
and the need for robust explanations: 
1. ability to do configurable explanation for 
 subset of test nodes; 
2. ability to generalize same structure 
upon small, arbitrary changes of k edges 
to the graphs; 3. ability to support querying.}
}
\end{example}

\eat{Existing 
approaches characterize 
explanations as 
input features 
directly from \gnn layers 
that are coherent with results, and remain limited to retrieving useful 
structural information as needed~\cite{yuan2022explainability}.}

\vspace{.2ex}
Our second example illustrates the need for 
such structures in cyber security, where 
new edges are inserted as part of   
real-world cyber attacks. In such context, 
the explanation structures themselves 
may indicate ``invariant'' structures 
that need to be 
protected {\em regardless} of 
attack tactics. 

\vspace{.2ex}
\begin{example} \textbf{[Understanding 
``Vulnerable Zone'' in Cyber Networks]}. 
Graph $G_2$ in Fig. \ref{fig-motivation} is a fraction of a provenance graph~\cite{ding2023case}. It consists of nodes that represent files (oval) or processes (rectangle), while the edges depict access actions. The graph 
contains a substructure that captures a multi-stage attack. The attack is initiated by an email with a malicious attachment that is disguised as an invoice, and 
targets at a script ``breach.sh'' to perform a data breach.  Attack tactics are encoded as paths or subgraphs (with red-colored, solid or dashed edges) in $G_2$. 
A \gnn-based node classifier is trained to identify potential attack targets (to be 
labeled as ``vulnerable'', colored orange in $G_2$)~\cite{bilot2023graph}.
\eat{Graph $G_2$ in Fig. \ref{fig-motivation} illustrates a 
fraction of a provenance graph. A provenance graph~\cite{ding2023case} 
is a common network representation for forensics analysis, which contains nodes as files or processes, 
and edges indicating access actions. Attack tactics 
can be encoded as paths or subgraphs in a 
provenance graph. A \gnn-based node classification 
task labels a set of test nodes to be ``vulnerable'' 
as potential attack targets or not~\cite{bilot2023graph}, given 
labeled training examples from known 
attacking events. 
}

\vspace{.5ex}
An emerging challenge is to effectively detect 
a two-stage tactic, which first uses, \eg ``DDoS'' 
as a deceptive attack (via red dashed edges) on various fake targets to exhaust defense resources, 
to hide the intention of data breaching from an 
invariant set of high-value files as true target (via red solid edges). In such cases, a \gnn-classifier 
can be ``fooled'' to identify a test node as ``vulnerable'' 
simply due to majority of its neighbors wrongly 
labeled as ``vulnerable'' under deceptive stage. 
On the other hand, a robust explanation for those correct labeling of test nodes should stay the same  
for a set of training attack paths. In such cases, 
the tactic may switch the first-stage targets from time to time to maximize 
the effect of deceptive attacks. As the explanation  
structures remain invariant despite various 
deceptive targets, 
 they may reveal the true intention regardless of 
how the first stage deceptive targets change.   
As such, the explanation also suggests 
a set of true ``vulnerable'' files that {\em should} be 
protected, thereby helping mitigate the impact of 
deceptive attacks. 

Nevertheless, computing such an explanation can be 
expensive due to the large size of 
provenance graphs. For example, a single 
web page loading may already result in 22,000 system calls and yields a provenance 
graph with thousands of nodes and edges~\cite{hassan2020tactical,ding2023case}. 
Especially for sophisticated multi-stage attacks above, it may leave complex and rapidly changing patterns in the provenance graph. Detecting and mitigating such attacks require the ability to process large amounts of data quickly for real-time or near-real-time threat analysis.
\end{example}

The above examples illustrate 
the need to generate subgraphs 
as explanation 
structures for \gnn-based classification, 
which are (1) \textit{``factual''}, \ie
contributing subgraphs 
that are consistently responsible for 
the same assigned class labels to 
test nodes, 
(2) \textit{``counterfactual''}, \ie\ 
 \gnn result changes if such explanation 
structures are removed, 
and moreover, 
(3) \textit{``robust''}, \ie 
remain to be a factual and 
counterfactual explanation 
upon certain amount of  
 structural disturbance 
in the graphs. 
Several methods generate explanation
structures (factual or counterfactual)
for \gnn-based classification~\cite{ying2019gnnexplainer,yuan2021explainability,huang2023global,zhang2022gstarx,yuan2020xgnn}, 
yet can not guarantee 
all the three (factual, 
counterfactual, and robust)
criteria upon edge updates. 

Given a pre-trained 
\gnn model $\M$ and a 
graph $G$, {\em how do we formally 
characterize such ``invariant'' structures?} 
Moreover, {\em how can 
we efficiently compute 
such ``invariant'' structures for 
\gnn-based node classification 
over large graphs? }

\stitle{Contributions.} 
To address the above challenges, 
we propose a novel class of \gnn 
explanation structures called 
{\em robust counterfactual witness}, 
and develop cost-effective 
methods to compute them. 
We summarize our main contributions 
below. 

\sstab
(1) We formalize the notion of 
{\em robust counterfactual witness} (\rcw) for 
\gnn-based node classification. 
An \rcw is a subgraph 
in $G$ that preserves the 
result of a \gnn-based classifier if 
tested with \rcw alone, 
remains counterfactual 
(i.e., the \gnn gives different output 
for the remaining fraction of the graph 
if \rcw is excluded), and 
robust (such that it can preserve 
the result even if 
up to $k$ edges
in $G$ are changed). 

\sstab 
(2) 
We analyze properties of 
$k$-\rcws in regards to 
the quantification of 
\textit{$k$-disturbances}, where 
$k$ refers to a budget 
of the total number of 
node pairs that are
allowed to be disturbed. 
We formulate the problems  
for verifying (to decide if a subgraph is a 
$k$-\rcw) and generating \rcws 
(for a given graph, \gnn, and 
$k$-disturbance)
as explanations for \gnn-based node 
classification. 
We establish their 
computational hardness results, 
from tractable cases to co-NP hardness. 

\sstab 
(3) We present
feasible algorithms to 
verify and 
compute \rcws, matching 
their hardness results. 
Our verification algorithm 
adopts an 
``\textit{expand-verify}''
strategy to iteratively 
expand a $k$-\rcw structure
with node pairs that 
are most likely to change 
the label of a next test node, 
such that the explanation 
remains stable by including 
such node pairs to 
prevent the impact of the disturbances. 
\eat{We introduce effective 
data structures to 
improve the efficiency, 
which dynamically 
stores a set of 
``disapproval'' 
structures at each layer, 
in a memory-efficient manner, 
and grows the disapprovals 
to gradually extend to
\rcws. }

\sstab
(4) For large graphs, 
we introduce a parallel algorithm to 
generate $k$-\rcws at scale. 
The algorithm 
distributes the 
expansion and 
verification process to partitioned 
graphs, and by synchronizing  
verified $k$-disturbances, dynamically 
refines the search area of 
the rest of the subgraphs for 
early termination. 

\sstab
(5) Using real-world datasets, 
we experimentally 
verify our algorithms. 
We show that it is feasible 
to generate \rcws 
for large graphs. 
Our algorithms can produce  
familiar structures for domain experts, as well as for large-scale classification tasks. 
We show that they are 
robust for noisy graphs. We also demonstrate application scenarios 
of our explanations.

\stitle{Related Work}.
We summarize the related work as follows. 

\etitle{Graph Neural Networks}. 
Graph neural networks (\gnns) are deep learning models designed to tackle graph-related tasks in an end-to-end manner~\cite{wu2020comprehensive}. 
Notable variants of \gnns include graph convolution networks (GCNs)~\cite{kipf2016semi}, attention networks (GATs)~\cite{velivckovic2017graph}, graph isomorphism networks (GINs)~\cite{xu2018how}, APPNP~\cite{gasteiger2018predict}, and GraphSAGE~\cite{hamilton2017inductive}). 
Message-passing based \gnns 
share a similar feature learning paradigm: for each node, update node features by aggregating neighbor counterparts. 
The features can then be converted 
into labels or probabilistic scores for 
specific tasks. 
\gnns have demonstrated their efficacy 
on various tasks, including node and graph classification~\cite{kipf2016semi,xu2018how, zhang2018end, ying2018hierarchical, du2021multi}, and link prediction~\cite{zhang2018link}. 

\etitle{Explanation for GNNs}. 
Several approaches have been proposed to generate explanations for 
\gnns\cite{zeiler2014visualizing, schwarzenberg2019layerwise, huang2022graphlime, schlichtkrull2020interpreting, luo2020parameterized, ying2019gnnexplainer,cf-gnnexplainer, yuan2021explainability, yuan2020xgnn, huang2023global,zhang2022gstarx}. 
For example, GNNExplainer~\cite{ying2019gnnexplainer} learns to optimize soft masks for edges and node features to maximize the mutual information between the original and new predictions and induce important substructures from the learned masks. SubgraphX~\cite{yuan2021explainability} employs Shapley values to measure a subgraph's importance and considers the interactions among different graph structures. XGNN~\cite{yuan2020xgnn} explores high-level explanations by generating graph patterns to maximize a specific prediction. GCFExplainer ~\cite{huang2023global} studies the global explainability of GNNs through counterfactual reasoning. Specifically, it finds a small set of counterfactual graphs that explain \gnns.
CF2~\cite{tan2022learning} considers a learning-based approach to infer 
substructures that are both factual and 
counterfactual, yet does not consider 
robustness guarantee upon graph disturbance.  
\cfexplainer~\cite{cf-gnnexplainer} generates counterfactual explanations for GNNs via minimal edge deletions. 
Closer to our work is the generation of 
robust counterfactual explanations for \gnns~\cite{bajaj2021robust}. The explanation structure is required to be 
counterfactual and stable (robust), 
yet may not be a factual explanation. 
These methods do not explicitly support all three of our objectives together: robustness,
counterfactual, and factual explanations; neither scalable explanation generation for large graphs are discussed. 

\etitle{Robustness of GNNs}. 
Robustness of \gnn models has been separately 
studied, such as  
certifiable robustness~\cite{bojchevski2019certifiable}. It verifies if 
a model remains robust under bounded disturbance of 
edges. 
In particular, 
robust training of \gnns that ensures such 
robustness has been investigated~\cite{guan2022robognn}. 
A general strategy is to identify 
adversarial edges that, if 
removed, changes the results of 
test nodes. If no such structure 
can be identified, the nodes are certified 
to be robust. 
The key difference in our problem is that we consider robust substructures of the data that preserve the results of a given fixed model as well as its explanation structure. 
Therefore, prior robust learning 
cannot be used to identify 
our explanation structures due to different objectives. 


\eat{
\vspace{-3ex}
\etitle{Graph Views}. Graph views have been studied as a 
useful approach to access and query large graphs~\cite{mami2012survey}. A graph view 
consists of a graph pattern and a set of subgraphs as its 
matches via graph pattern matching. Graph views are shown to be 
effective in view-based query processing~\cite{fan2014answering}, 
summarization~\cite{song2018mining}, event analysis~\cite{zhang2020distributed}, 
query suggestion~\cite{ma2022subgraph}, 
data cleaning~\cite{lin2019discovering}
and data pricing~\cite{chen2022gqp}. Several approaches 
have also been developed to discover 
graph views~\cite{song2018mining,lin2019discovering}. 

\vspace{.5ex}
To the best of our knowledge, this is the first work that exploits graph views to support  
queryable explanation  
for \gnn-based classification. Our approach 
is a post-hoc method 
that treats \gnns as 
black-box (hence does not require details 
from \gnns, but only the output 
from its last layer), does not 
require node/edge mask training, 
and generates explanations 
as views that are queryable, 
concise, and class label-specific, 
all in a user-configurable manner 
(Table~\ref{tab:comprehensive-analysis}). 
}

\section{Preliminaries}
\label{sec:pre}
\vspace{-0.5ex}
%
%

We start with a review of \gnns. We then introduce our robust subgraph explanation structures. Important notations are summarized in Table~\ref{tab:notations}.

\begin{table}[tb!]
\begin{scriptsize}
\caption{Symbols \& Notations}
\label{tab:notations}
\vspace{-2ex}
\begin{tabular}{|c|p{0.3\textwidth}|}
\hline
\textbf{Symbol} & 
\textbf{Meaning} \\
\hline 
$G$ = $(V, E)$ & graph with nodes $V$ and edges $E$ \\
\hline 
$(X, A)$ & $X$: node feature matrix; $A$: adjacency matrix \\
\hline
$\M$; $M$ & a GNN-based classifier; inference 
function of $\M$ \\
\hline
$X^i$; \textbf{Z} & node embeddings at layer $i$ of $\M$; \newline 
logits (output) of inference process \\
\hline
$L$; $F$ & \# layers of $\M$; \# features per node \\
\hline
$G_w$; $G_s$ & a verified witness; a subgraph to be 
verified \\
\hline
$C$=$(G,G_s,V_T,M,k)$ & a configuration that specifies graph $G$; \newline
a subgraph $G_s$ to be 
verified as witness; \newline
inference function $M$ (of a 
\gnn $\M$); \newline test nodes $V_T$, and $k$ (in $k$-disturbance)\\
\hline
\cw; \rcw; $k$-\rcw & counterfactual witness; \newline 
robust counterfactual witness; \newline $k$-robust counterfactual witness \\ 
\hline
$m^*_{l,c}(v)$ & worst-case margin of test node $v$ \\
\hline 
\eat{
\dq{$G$; $\V_T$} & \begin{tabular}{@{}c@{}} \dq{The graph that contains all test nodes;} \\ \dq{test nodes of $G$ for classification} \end{tabular} \\
\hline 
\hline 
\dq{$G_w = (V_w, E_w)$} & \dq{A subgraph $G_w$ containing all test nodes $V_T \subseteq V_w$} \\
\hline 
\dq{CW} & \dq{A counterfactual witness if $\M(v,G\setminus G_w)$ $\neq l$} \\
\hline 
\dq{$k$-RCW} & \dq{A robust counterfactual witness \wrt $k$ robustness. }
\hline
\dq{$\C$ = ($\theta$, $k$)} &
\begin{tabular}{@{}c@{}} 
\dq{A configuration that specifies} \\
\dq{explainability $\theta$ and local budget $k$}
\end{tabular}\\
}
\hline
\end{tabular}
\end{scriptsize}
\vspace{-2ex}
\end{table}

\vspace{-1.5ex}
\subsection{Graphs and Graph Neural Networks}
\label{sec:gnns}

\vspace{-1.5ex}
\stitle{Attributed Graphs}.  
We consider a connected 
graph $G = (V,E)$, where $V$ is the set of nodes, and $E\subseteq V\times V$ is the set of edges. 
Each node $v$ carries a tuple ${\sf F}(v)$ 
of \st{node} attributes (or features) and their values. 

\eat{
For example, the Graph Convolution Network (\gcn) \cite{kipf2016semi}, 
a representative \gnn model, adopts a general form 
as:
\begin{equation}
\label{eq-prop}
    X^k = \delta(\Hat{D}^{-\frac{1}{2}} \Hat{A} \Hat{D}^{-\frac{1}{2}} X^{k-1}  \Theta^{k})
\end{equation}

Here $\Hat{A} = A + I$, where $I$ represents the identity matrix and $A$ is the adjacency matrix of graph $G$. $X^k$ indicates node feature representation in the $k$-th layer, (with $X^0=X$ a matrix of input node features), where each row $X_v$ is a vector (numerical) encoding of a node tuple $A(v)$. 
The encoding can be obtained by, e.g., 
word embedding or one-hot encoding~\cite{gardner2018allennlp}. $\Hat{D}$ represents the diagonal node degree matrix of $\Hat{A}$, $\delta(.)$ is the non-linear activation function, and $\Theta^{k}$ represents the learnable weight matrix for the $k$-th layer.
}


\stitle{Graph Neural Networks (GNNs)}. 
\gnns~\cite{wu2020comprehensive} are a family of well-established deep learning models that extend traditional 
neural networks to transform graphs into proper 
embedding representations for various downstream analyses such as node classification. 
\gnns employ a multi-layer message-passing scheme, through which the feature representation of a node $v$ in the next layer is aggregated from its 
counterparts of the neighborhoods of $v$ in $G$ at the current layer.
A \gnn with $L$ layers iteratively gathers and aggregates 
information from neighbors of a node $v$ to compute 
node embedding of $v$. 
For example, the Graph Convolution Network (\gcn) \cite{kipf2016semi}, 
a representative \gnn model, adopts a general form 
as:
\begin{equation}
\label{eq-prop}
    X^i = \delta(\Hat{D}^{-\frac{1}{2}} \Hat{A} \Hat{D}^{-\frac{1}{2}} X^{i-1}  \Theta^{i})
\end{equation}

Here $\Hat{A}=A+I$, where $I$ is the identity matrix, and $A$ is the adjacency matrix of $G$. $X^i$ indicates node feature representation in the $i$-th layer ($i \in [0, L]$).  Here $X^0=X$, which refers to the matrix of input node features. Each row $X_v$ of $X$ is a vector encoding of a node tuple ${\sf F}(v)$.\footnote{ 
The encoding can be obtained by, e.g., 
word embedding or one-hot encoding~\cite{gardner2018allennlp}, among other featurization techniques.} The matrix $\Hat{D}$ represents the diagonal node degree matrix of $\Hat{A}$. 
The function $\delta(.)$ is a non-linear activation function, and $\Theta^{i}$ represents the learnable weight matrix at the $i$-th layer.

\ppnps, a class of Personalized PageRank based \gnns $\M$~\cite{bojchevski2019certifiable}, specifies 
the $\delta$ function (Eq.~\ref{eq-prop})  with a Personalized PageRank propagation matrix: 
$X^L =(1-\alpha)(I-\alpha\Hat{D}^{-1} \Hat{A})^{-1}\cdot X\Theta$. 
Here  $D$ is the diagonal matrix of node out-degrees with ${D}_{i i}=\sum_{j} {A}_{i j}$, and $\alpha$ is teleport probability. 
Other notable \gnn variants are 
\gsage~\cite{hamilton2017inductive} 
that samples fixed-size neighbors, 
and \gat~\cite{velivckovic2017graph} that 
incorporates self-attention.  

\eat{Denote the output features 
$\textbf{h}^i_v$ (with $v$ ranges over $V$) at layer 
$i$ as $\textbf{X}^i$. A \gnn typically computes 
$\textbf{X}^i$ as 
\begin{equation}
\label{eq:prop}
\textbf{X}^i = \delta (\parallel^n_{j=c}\tilde{A}^j \textbf{X}^{i-1}\textbf{W}^i_j)
\end{equation}
\textcolor{blue}{where $\parallel$ denotes the horizontal concatenation operation} 
$\textbf{W}^i_j$ refers to the learnable weight matrix of order $j\times j$ 
at layer 
$i$, $\delta(\cdot)$ is an activation, and 
$\tilde{A}$ is a normalized adjacency matrix, 
Notable \gnn variants are 
\gcn and \gsage~\cite{hamilton2017inductive} 
that sample fixed-size neighbors 
and \gat~\cite{velivckovic2017graph} that 
incorporates self-attention 
for neighbors. 
where  $\boldsymbol{\Pi}$ is a PageRank matrix, and $\textbf{X}^L$ is the output features from the last layer of the GNN.
}

\eat{
\eetitle{Training}. 
Given a set of labeled training nodes $V_{Tr}$,  
a \gnn-based classifier $\M$ outputs ``logits'' ${Z \in \mathbb{R}^{|V| \times |L|}}$  
that are fed to a softmax layer to be transformed to $Z'$ that encodes the probabilities of assigning 
a class label to a node. 
For example, the training of $M$ may minimize a loss function 
$\mathcal{L}(Z, A)$ = $-\sum_{v\in V_{Tr}} l \ln Z'_v$,
where $l$ is the true label of a training node $v\in V_{Tr}$, 
and $Z'_v$ is  
the embedding of a training node $v$ in $V_T$.
\textcolor{red}{[AK: However, we defined $Z'$ as encoding of probabilities and not node embeddings.]}
$\mathcal{L}$ can also be specified to minimize 
a task-specific loss. 
}

\stitle{\gnn-based node classification}.  
Given a graph $G$ = $(V,E)$ and a set of labeled 
``training'' nodes 
$V_{Train}$,  
the node classification task is to learn a 
{\em \gnn-based classifier} $\M$
to infer the labels of a set of unlabeled 
test nodes. A \gnn-based classifier   
$\M$ of $L$ layers (1) takes as input $G$=$(X,A)$ and learns to generate representations  
$X^L$, converts them into labels that best fit the labels of $V_{Train}$ in $G$, and (2) assigns, for  
each unlabeled test node $v\in V_T$, a class label $l\in\text{\L}$. 

\eetitle{Inference}. The label assignment is determined 
by an ``inference process'' $M$ of $\M$.  
Given a set of test nodes 
$V_T$, the process outputs a ``logits'' matrix ${Z \in \mathbb{R}^{|V_T| \times |L|}}$ 
in $M$'s output layer, for each node $v\in V_T$ 
and each label $l\in L$, to be converted 
to likelihood scores of 
$l$ given $v\in V_T$; the higher, the more likely 
$v$ is assigned with $l$. A softmax 
function is then used to 
convert the logits to corresponding 
labels. To make a (simplified) difference, 
(1) we consider the inference process as a 
polynomial-time computable function   
$M(v,G)$ = $l$ ($v\in V_T$), 
with a ``result'' $l$ for $v$; 
and (2) we refer to the 
output logits $Z$ of 
$M$ simply as ``output'' of $M$. 

In particular, we consider three 
trivial cases: $M(v,v)$ = $l$, 
and denote $M(v, \emptyset)$ 
or $M(\emptyset, G)$ as ``undefined''. 

\eetitle{Fixed and deterministic GNN}. 
Given a set of test nodes $V_T$, 
a \gnn $\M$ is {\em fixed}
if its inference process $M(v,G)$ does not change 
for any input test node $v$ in $V_T$. 
That is, it has all factors 
which determine the computation of $M(\cdot)$ 
such as layers, 
model parameters, among 
others, fixed. 
We say that $\M$ is {\em deterministic} 
if $M(\cdot)$ always generates the same output 
for the same input. 
In this work, 
we consider a fixed, deterministic 
\gnn $\M$. In practice, most of 
the \gnns are 
deterministic for the need of 
consistency and accuracy guarantees.  

\eat{
\stitle{Remarks}. We remark on the difference between 
inference process of $M$ that ``outputs 
probabilistic scores'' and ``outputs in a 
stochastic manner''. 
A \gnn may take a stochastic process $M(\cdot)$, 
that is not ``deterministic'' as 
it outputs varies even with 
the same input. 
}

\eat{
$\mathcal{G} = \{G_1, G_2, \dots, G_m\}$
and a set of class labels $\L$, 
a {\em \gnn-based classifier}  
$\M$ of $k$ layers (1) takes as input \st{their feature representations} $G_i$=$(X_i,A_i)$ $(i\in [1,m])$, 
\st{and} learns to generate representations  
$X_i^k$, \st{and} converts them into labels that best fit a set of labeled 
graphs (``training examples''), and (2) assigns, for  
each unlabeled ``test'' graph $G_i\in \mathcal{G}$, a class label $l\in\L$ 
(denoted as $\M(G_i)$ = $l$). 
}

\subsection{Explanation Structures}
\label{sec-expstructure}

To characterize our explanation structures 
as subgraphs, we introduce a 
notation of {\em witnesses}. 

\stitle{Witnesses}. 
Given a graph $G$, class labels $\text{\L}$, and a \gnn $\M$, 
a test node $v\in V_T$ for which $M(v,G)$=$l\in \text{\L}$, 
we say that a 
subgraph $G_w$ of $G$ is 

\sstab
(1) a {\em factual witness} of  
result $M(v,G)$=$l$, 
if $M(v,G_w)=l$\footnote{Note that this indicates $G_w$ containing $v$ as 
one of its nodes.}; 

\sstab 
(2) a {\em counterfactual witness} (\cw) 
of $M(v,G)$=$l$,  
if (a) it is a factual witness of $M(v,G)$=$l$, and 
(b) $M(v,G\setminus G_w)$ $\neq l$; 
and 

\sstab 
(3) a {\em $k$-robust counterfactual witness} $k$-\rcw of 
$M(v,G)$=$l$, if for {\em any} graph $\widetilde{G}$ 
obtained by a {\em $k$-disturbance} 
on $G\setminus G_w$, 
$G_w$ remains to be a 
\cw for $M(v,\widetilde{G})$=$l$.  

\vspace{.5ex}
Here (1)  $G\setminus G_w$ is the
graph obtained by removing 
all the edges of $G_w$ from $G$, 
while keeping all the nodes; and 
(2) a {\em $k$-disturbance} to a graph $G$ refers to an operation 
that ``flips'' {\em at most} $k$ node pairs (denoted as 
$E_k$) from $G$, including 
edge insertions and removals. 
For a $k$-disturbance posed on 
$G\setminus G_w$, it only ``flips'' a set of 
node pairs $E_k$ that are not in $G_w$, 
\ie it does not insert nor remove edges of $G_w$. 
A $k$-disturbance may capture a natural 
structural difference 
of graphs (\eg  chemical compounds that 
differ with at most $k$ bonds).  

\vspace{.5ex}
Given a set of test nodes $V_T\subseteq V$, 
we say that $G_w$ is a $k$-\rcw of $V_T$ \wrt\ 
\gnn $\M$, 
if for every test node $v\in V_T$, 
$G_w$ is a $k$-\rcw for the result $M(v,G)$=$l$. Note that 
a $k$-\rcw $G_w$ of the test nodes $V_T$ 
naturally contains $V_T$ as its nodes.

We define several 
{\em trivial cases}. 
(1) 
A single node $v\in V_T$ is a {\em trivial} factual witness 
of $M(v, G)$=$l$, as $M(v,v)$=$l$ trivially; 
it is a {\em trivial} \cw because $M(\emptyset,G)$ 
is undefined.   
(2) The graph $G$ is a trivial factual witness 
as $M(v, G)$ trivially equals to itself;   
a trivial \cw for $M(v, G)$=$l$, 
as it is a trivial factual witness,  
and $M(v,\emptyset)$ is undefined; and also 
a trivial $k$-\rcw, since 
no $k$-disturbance can be 
applied to $G\setminus G$ = $\emptyset$. 
In practice, we aim at finding ``non-trivial'' $k$-\rcws  
as 
subgraphs with at least an edge, and are not $G$ themselves. 

\begin{example}
\label{exa-rcw}
\eat{
Consider graph $G_1$ in Fig.~\ref{fig-motivation}, (1) on the left graph, the bold structure indicates a nitro group, the prediction of the model with only the nitro group is same as the molecule itself, since the nitro group is the key structure to trigger mutagenicity. This behavior of the nitro group indicates that it is a factual witness. Meanwhile, the remaining graph by removing the nitro group will result in an opposite prediction, since current atoms are stable with a maximum number of connected bonds. Therefore, the nitro group is also a counterfactual witness (\cw). (2) On the middle graph, we show a similar molecule variant by removing two edges from the original molecule, the variant occurs with one additional mutagenic component, which is the red bold structure, i.e., the aldehyde structure. If we also want to obtain witness and \cw for the variant, it will result in two drawbacks: we need to generate it by executing the explainability method again with additional time cost; and different explanation structures may confuse the domain expert while interpreting a group of molecule variants. (3) Therefore, it is critical to explain them with a unified explanation structure, generated at once. On the right graph, we show a robust counterfactual witness, indicated as the bold structure, to tackle the drawbacks: instead of generating explanations by considering only one molecule, we consider a group of similar molecules indicated by chemical bonds' connection differences. 
}
The graph $G_2$ (Fig.~\ref{fig-motivation}) has a test node `breach.sh', which is the target file for a data breach and is labeled ``vulnerable''. 
A valid attack path must access a privileged file (`/.ssh/id$\_$rsa' or `/etc/sudoers') and the command prompt (`cmd. exe') beforehand. 
Prior to executing the true attack path (via red solid edges), the malware performs a non-aggressive deceptive `DDoS' attack (via red dashed edges) to mislead the prediction of $M$, where the GNN $\M$ is trained to identify whether a file is ``vulnerable'' (in orange) or not (in green).

\sstab
(1) There are two primary {\em witnesses} to explain the test node's label: (`cmd.exe', `/.ssh/id$\_$rsa', `breach.sh') and (`cmd.exe', `/etc/sudoers', `breach.sh'), each includes a factual path that leads to the breach and serves as evidence to explain the test node's label as ``vulnerable''. 

\sstab
(2) Furthermore, a \cw is shown as the combined subgraph of these two witnesses, denoted as $G_{w_2}$ (shaded area of $G_2$ in Fig.~\ref{fig-motivation}). Removing all edges in $G_{w_2}$ from $G$ will reverse the classification of the test node to ``not vulnerable'', as the remaining part is insufficient for a breach path. 

\sstab
(3) $G_{w_2}$ is also a $k$-\rcw for the test node's label, where $k$=$3$, which is the maximum length of a deceptive attack path. It indicates that for any attack path involving up to $k$ deceptive steps, $G_{w_2}$ remains unchanged. It, therefore, identifies the true important files which must be protected to mitigate the impact of deceptive attacks. 
\end{example}

We further justify that 
\rcws serve as a desirable 
explanation structure 
by showing 
an invariant property. 
\eat{
\textcolor{red}{[AK: Another property could be to prove that given graph $G$, a set of test nodes 
$V_T$ and a fixed deterministic \gnn $M$, 
a $k$-\rcw for $V_T$ always exists. I think it is important to discuss, since otherwise what is the guarantee that Alg-1 input will always have some output?]}
}

\begin{lemma} 
\label{lm-k}
Given graph $G$, a set of test nodes 
$V_T$, and a fixed deterministic \gnn $\M$ with inference function $M$, 
a $k$-\rcw for $V_T$ is a $k'$-\rcw for $V_T'$, for any $k'\in[0,k]$, 
$V_T'\subseteq V_T$. 
\end{lemma}

\vspace{-1mm}
\begin{proof}
We first prove that 
for any node $v\in V_T$, 
a $k$-\rcw $G_w$ for $M(v,G)$ 
remains to be a $k'$-\rcw 
for $M(v,G)$, for any $k'\in[0,k]$. 
One can verify the above result by contradiction. 
Assume that there is a $k'\in [0,k]$, such that 
a $k$-\rcw $G_w$ is not a $k'$-\rcw, then 
(a) $G_w$ is not a factual witness, or 
it is not counterfactual, both contradict 
to that $G_w$ is an \rcw; 
or (b) $G_w$ is a \cw but is not robust. Then there exists a $k'$-disturbance with $k'$ edges $E_k'$ in 
$G\setminus G_w$, which 
``disproves'' that $G_w$ is a $k$-\rcw. As $k'\leq k$, 
the same $k'$-disturbance prevents $G_w$ to be a $k$-\rcw by definition. 
This contradicts to that $G_w$ is a $k$-\rcw. 
Hence, $G_w$ remains to be a 
$k'$-\rcw for $M(v,G)$. 
As the test node $v$ ranges over 
$V_T$, the above analysis holds 
to verify that $G_w$ remains to be a 
$k'$-\rcw for any subset 
$V_T'\subseteq V_T$. 
Lemma~\ref{lm-k} thus follows. 
\end{proof}
\eat{
\vspace{.5ex}
\begin{lemma} 
\label{lm-anti}
Given graph $G$, a test node 
$v\in V_T$, and a \gnn $M$, 
for any subgraph $G_s$ 
of $G$ that is not 
a \cw for 
$M(v, G)$=$l$, $G_s$ is not 
a $k'$-\rcw of $M(v, G')$=$l$,  
for any $k'\geq k$ and any graph $G'$ 
with result $M(v, G')$=$l$, 
where $G'$ is 
obtained by inserting any edges to $G$ 
without changing its node set. 
\end{lemma}
\begin{proof} 
Let $G'$ be a graph obtained by 
inserting an arbitrary set of edges to $G$. 
As $G_s$ is not a $k$-\rcw of 
$M(v, G)$, we consider the 
following cases. 
(1) $G_s$ is 
not a factual 
witness of $M(v,G)$=$l$. For 
$G'$ obtained by edge insertions, 
then it is not a $k$-\rcw by definition; 
(b) if $G_s$ is a 
factual witness in $G$, 
\ie $M(v, G)$ = $l$, 
there are further two cases. 
(b.1) $G_s$ is not a \cw in $G$. 
In this case, it is not a 
$k$-\rcw for $G'$. 
(b.2) $G_s$ remains to be a 
\cw in $G'$ 

but not a $k$-\rcw 
in $G$. Thus there must exist a 
$k$-disturbance that removes a set of 
edges $E_k$ and prevents $G_s$ to be a $k$-\rcw 
for $V_T$ \wrt $\M$. 
Observe that $E_k\subseteq 
G'\setminus G_s$. \textcolor{blue}{Hence, 
removing $E_k$ 
already prevents $G_s$ 
to be a $k$-\rcw in $G$} \textcolor{red}{[AK: Why? Does not CW in $G$ ensure not CW in $G'$?]}. 
Thus, $G_s$ is not a $k$-\rcw 
in $G'$ by definition. 
(2) $G_s$ is a factual witness  
but not a \cw 
in $G$. Following the same  
analysis as above from 
case (b.1), $G_s$ is not a $k$-\rcw 
in $G'$. 
(3) $G_s$ is a \cw 
but not a 
$k$-\rcw. Following the same analysis 
starting from case (b.2), 
the same result holds. 

Putting these together, 
Lemma~\ref{lm-anti} follows. 
\end{proof} \textcolor{red}{[AK: Is it possible to add a few sentences why these lemmas would be helpful in the later part of this paper?] -- [Wu: will add these after I did my pass 
of the algorithm, to decide what to be best put here.]}
}

We next study two classes of 
fundamental problems for \rcws, notably, verification and generation. 


\section{Verification of Witnesses}
\label{sec-expsub}

For the ease of presentation, 
we introduce an input 
configuration for verification problem. 

\stitle{Configuration}. 
A configuration $C$ = $(G, G_s, V_T, M, k)$ 
specifies a graph $G$ with a set of 
test nodes $V_T$, a fixed, deterministic 
\gnn inference process $M$, an integer $k$ 
to specify $k$-\rcw, 
and a subgraph $G_s$ of $G$. 
Specifically, a counterfactual witness \cw is a $0$-\rcw ($k$=0), \ie 
no robustness requirement is posed. 

\subsection{Verification Problem}
\label{sec-verify}

Given a configuration 
$C$ = $(G, G_s, V_T, M, k)$, 
the verification problem 
is to decide if $G_s$ is a 
$k$-\rcw for $V_T$ \wrt $M$. 
We first study the complexity of 
two special cases. 

\stitle{Verification of factual witnesses}. 
The witness verification problem  
takes as input a configuration 
$C$ = $(G, G_s, V_T, M, 0)$ 
and decides if $G_s$ is a 
factual witness for $V_T$, 
\ie for every test node $v\in V_T$, 
$G_s$ is a factual witness. 

We have the following result. 

\begin{lemma}
\label{lm-verifyw}
The witness verification problem is in \PTIME. 
\end{lemma}

\eat{
\begin{proof}
We provide a \PTIME algorithm, 
denoted as \verifyw, for 
the witness verification problem. 
By definition, it suffices to verify, for 
each node $v\in V_T$ with $M(v, G)$ = $l$, 
if $M(v, G_s)$ = $l$. To this end, 
\verifyw simply performs the inference process of \gnn $M$ over $G_s$, which takes as input the node features $X$, 
but a revised adjacency matrix 
with entries ``flipped'' to $0$ 
if the corresponding edges are not in $G_s$. 
As $M$ is a fixed deterministic 
model, the inference cost is in \PTIME. 
\end{proof}
}

Similarly, the verification problem  
for \cw is to decide if a given 
subgraph $G_s$ of $G$ is a 
counterfactual witness for $V_T$. 
We show that this does not 
make the verification harder. 

\begin{lemma}
\label{lm-verifycw}
The verification 
problem for \cw is in \PTIME. 
\end{lemma}

\eat{
\begin{proof}
We provide a second procedure, 
denoted as \verifycw, as a constructive 
proof. By definition, it performs the 
following. (1) It constructs 
a graph $G'$ = $G\setminus G_s$.  
(2) For each node 
$v\in V_T$ with assigned label $l$, 
(a) it first invokes \verifyw to 
verify if $G_s$ is a factual witness of $M(v,G)$, 
in \PTIME; 
and (b) If so, it invokes \verifyw to check if $M(v, G')\neq l$. 
If $G_s$ fails the test at any 
node $v\in V_T$, it returns no. 
Otherwise, $G_s$ is a \cw for 
$V_T$. As \verifyw is in \PTIME (Lemma~\ref{lm-verifyw}), 
\verifycw remains to be in \PTIME. 
\end{proof}
}

We prove the above two results with 
two verification algorithms, 
\verifyw and \verifycw, which 
invoke the inference process $M$ of 
\gnn $\M$ to check if the label 
$M(v, G_s)$ and $M(v, G\setminus G_s)$ 
remain to be $l=M(v, G)$, respectively. 
One may justify that 
\verifyw and \verifycw are in \PTIME by observing that
the inference function $M$ is \PTIME 
computable, with a cost typically 
determined by $L$, $|V|$, $|E|$, $d$, and $F$, 
where $L$ is the number of layers of the \gnn
$\M$, $d$ = $\frac{|E|}{|V|}$ is the average degree of $G$, and $F=|{\sf F}(v)|$ is the number of features per node.  
For example, the inference cost 
of message-passing based \gnns (\eg \gcns, \gats, \gsage), or PageRank based \ppnps is $O(L|E|F+L|V|F^2)$~\cite{chen2020scalable,
gasteiger2018predict}.


\vspace{.5ex}
Despite the verification of witnesses 
and \cw is tractable, its \rcw counterpart is nontrivial. 

\begin{theorem}
\label{lm-rwverify}
The $k$-\rcw verification problem is \NP-hard. 
\end{theorem}

\begin{proofS}
We make a case for \ppnps. 
Given $G_s$, we first 
invoke \verifyw and \verifycw in \PTIME to 
check if it remains to be a \cw. 
If $G_s$ is not a factual witness or a \cw, 
we assert that $G_s$ is not a $k$-\rcw. 
Otherwise, the problem 
equals to verify if 
a \cw $G_s$ is a $k$-\rcw. 
We prove this by first showing that it is {\em a general case} of the $k$-edge PageRank maximization problem: Find at most $k$ node pairs from a graph s.t. the PageRank scores of a targeted node $v$ is maximized if these $k$ links are inserted.
%

We then show the hardness of the
$k$-edge PageRank maximization problem by
constructing a \PTIME reduction from the link building problem (\lbp), a known \NP-hard problem~\cite{olsen2014approximability}: Given a graph $G=(V, E)$, a node $v \in V$, a budget $k \geq 1$, find the set $S \subseteq V \backslash\{v\}$ with $|S|=k$ s.t. the PageRank score of $v$ in $G'(V, E \cup(S \times\{v\}))$ is maximized.
As the $k$-edge PageRank maximization problem is \NP-hard, the verification 
of $k$-\rcw is \NP-hard. 
\end{proofS}

\eat{
\begin{algorithm}[tb!]
    \renewcommand{\algorithmicrequire}{\textbf{Input:}}
    \renewcommand{\algorithmicensure}{\textbf{Output:}}
    \caption{Algorithm \mrcwgen} 
    \begin{algorithmic}[1]
        \REQUIRE A configuration $C$ = $\{G, \emptyset, V_T, \M, k\}$, integer $b$ (local budget); \\
        \ENSURE A minimal $k$-\rcw $G_s$ of $V_T$ \wrt $\M$. \\
       \STATE $G_I$:=$\emptyset$; $G_s$:=$\emptyset$; set $E_k$:=$\emptyset$; set $V_u$:=$V_T$; integer $j$:=$1$;  
       \STATE $G_s$:= $V_T$;  $G_I$:=$G_s$; 
        \WHILE{$G\setminus G_s\neq\emptyset$}
            \STATE $G_s$:=$G_I$; $G_I$:= $\kw{Expand}(G_s, G, V_T)$; $j$:=$1$; 
            \FOR{$j = 1$ to $k$}
            \IF{$\verifyrcw(C, G_I, j)$ = \kw{true}}
                 \IF{j=$k$} 
                    \STATE \textbf{return} $G_I$; 
                 \ELSE  \STATE \textbf{continue}; 
                \ENDIF
            \ELSE \STATE $G_I$:=$G_s$; \textbf{break}; 
            \ENDIF
            \ENDFOR
        \ENDWHILE
       \STATE \textbf{return} $G_s$;     
    \end{algorithmic}
    \label{procedure:verifyrcw}
\end{algorithm}
}

\eat{
\begin{algorithm}[tb!]
    \renewcommand{\algorithmicrequire}{\textbf{Input:}}
    \renewcommand{\algorithmicensure}{\textbf{Output:}}
    \caption{Algorithm \verifyrcw (single node)} 
    \begin{algorithmic}[1]
        \REQUIRE A configuration $C$ = $\{G, G_s, v, \M, k\}$; \\
        \ENSURE  \kw{true} if $G_s$ is a 
        $k$-\rcw; \kw{false} otherwise. \\
         \IF{$\verifyw(C)$ = \kw{false}} 
        \STATE \textbf{return} \kw{false}; 
        \ENDIF
        \STATE construct graph $G'$ := $G\setminus G_s$; 
        \IF{$\verifycw(G',C)$ = \kw{false}}
        \STATE \textbf{return} \kw{false}; \ENDIF
        \STATE initialize bitmap $\hat{A}$ with adjacency matrix of $G'$;
        \FOR{$j = 1$ to $k$}
            \STATE 
            construct next distinct $j$-disturbance $E_k$; 
            \STATE
            flip $\hat{A}$ and update $G'$ (implicitly); 
            \IF{$\verifycw(G',C)$ = \kw{false}} 
            \STATE \textbf{return} \kw{false}; 
            \ENDIF
            \STATE restore $\hat{A}$ and $G'$; 
            \ENDFOR
            \eat{
            \IF{$\verifyrcw(C, G_I, j)$ = \kw{true}}
                 \IF{j=$k$} 
                    \STATE \textbf{return} $G_I$; 
                 \ELSE  \STATE \textbf{continue}; 
                \ENDIF
            \ELSE \STATE $G_I$:=$G_s$; \textbf{break}; 
            \ENDIF
            \ENDFOR
       \STATE \textbf{return} $G_s$; 
       }
      \STATE \textbf{return} \kw{true};  
    \end{algorithmic}
    \label{procedure:verifyrcw}
\end{algorithm}
}

\eetitle{Algorithm}. 
We outline a general verification 
algorithm, denoted as \verifyrcw.  
Given a configuration $C$, 
the algorithm first invokes \verifyw and \verifycw to 
check if $G_s$ is a factual witness and a \cw, 
both in \PTIME (Lemma~\ref{lm-verifyw} and Lemma~\ref{lm-verifycw}). 
If $G_s$ remains to be a \cw, 
it next performs a 
$k$-round of verification, 
to verify if $G_s$ is a
$j$-\rcw in the $j$-th round by enumerating all $j$-disturbances. 
It early terminates whenever 
a $j$-th disturbance is 
identified, which 
already disproves that 
$G_s$ is a $j$-\rcw, 
given Lemma~\ref{lm-k}. 

\vspace{-0.5ex}
\subsection{A Tractable Case} 

While it is in general intractable to verify $k$-\rcw, 
not all is lost: We show that 
there exists a tractable case 
for the verification problem 
when the $k$-disturbance is 
further constrained by a 
``local budget'', for  
the class of \ppnps. We 
introduce two specifications below. 

\eetitle{$(k, b)$-disturbance}. Given a graph $G$, 
a $k$-disturbance is a $(k, b)$-disturbance to $G$
if it disturbs at most $k$ edges, and meanwhile at most $b$ edges for each  
involved nodes in $G$. Here $b$ is a 
pre-defined constant as a local ``budget'' for a 
permissible disturbance. 
In practice, a local budget may 
indicate the accumulated total 
cost of ``local attack'' of a 
server, or variants of bonds in-degree 
as seen in chemical compound 
graphs~\cite{albalahi2022bond}, which are 
typically small constants (\eg $\leq 5$). 

\eetitle{Node robustness}. 
We start with a characterization of 
node robustness 
that extends 
certifiable robustness~\cite{bojchevski2019certifiable}, 
which verifies if a predicted 
label can be 
changed by a $k$-disturbance. 
This work makes a case for \ppnps, a class of 
Personalized PageRank based 
\gnns $\M$ as aforementioned. 
Given a configuration 
$C$ = $(G, G_s, V_T, M, k)$, 
where $M$ is the inference 
process for a fixed deterministic 
\ppnp $\M$, 
a node $v\in V_T$ is  
{\em robust} \wrt configuration $C$, 
if a ``worst-case margin'' 
$m_{l, *}^{*}(v)$ = $\min_{c \neq l}$ $m_{l, c}^{*}(v)>0$,
where $M(v, G)$ = $l$, 
and $c$ is any other label $(c\neq l)$.  
We define
$m_{l, c}^{*}(v)$ as:
\eat{
\[
m_{y_{t}, c}^{*}(v)=\min _{\tilde{G} \in G \cup E'} m_{y_{t}, c}(v)\\
=\min _{\tilde{G} \in G \cup E'} \pi_{\tilde{G}}\left(v\right)^{T}\left(Z_{\{:, y_{t}\}}-Z_{\{:, c\}}\right)
\]
}
\begin{equation}
\label{eq-margin}
\begin{split}
\vspace{-2.5ex}
m_{l, c}^{*}(v) &=\min _{E_k \subseteq G\setminus G_s} m_{l, c}(v)\\
& =\min_{E_k \subseteq G\setminus G_s} \pi_{E_k}\left(v\right)^{T}\left(Z_{\{:, l\}}-Z_{\{:, c\}}\right)
\vspace{-1ex}
\end{split}
\end{equation}
\eat{~\footnote{Notice that here the ``robustness'' is not a direct guarantee for $M$ to assign 
``correct'' label. Rather to simplify 
our discussion, we assume a high quality \gnn $\M$ with 
inference process $M$ that assigns the true 
label for node $v$.} }
where $E_k$ ranges over all the possible 
$(k,b)$-disturbances that can be applied to 
$G\setminus G_s$, and ${\pi_{E_k}\left(v\right) =\boldsymbol{\Pi}_{v,:}}$ is the PageRank vector of node $v$ in the PageRank matrix $\boldsymbol{\Pi}=(1-\alpha)({I}_{N}-\alpha {D}^{-1} {A'})^{-1}$, 
with $A'$ obtained by 
disturbing the adjacency matrix of $G$ 
 with $E_k$, \eg $A'[i,j]$ = $0$ (resp. $1$) if 
$(v_i, v_j)\in E_k$ (resp. $(v_i, v_j)\not\in E_k$). 
Note that we 
do not explicitly remove the edges and change $G$, 
but reflect the tentative disturbing 
by computing $A'$. ${Z \in \mathbb{R}^{|V| \times |L|}}$ collects the individual per-node logits for \ppnps.
By verifying $m_{l, *}^{*}(v) > 0$ (\ie $m_{l, c}^{*}(v) > 0$ for any $c\in \text{\L}$ 
other than the label $l=M(v,G)$), it indicates that under any disturbance of size $k$ 
in $G\setminus G_s$, $M$ always predicts the 
label of node $v$ as $l$. 

\eat{
\begin{theorem}
\label{lm-rwverify-tractable}
Given the configuration $C$ that specifies 
\ppnp $\M$, and when only $(k,b)$-disturbance is 
allowed, the verification problem for $k$-\rcw is in 
\PTIME.  
\end{theorem}
}

We next outline a \PTIME 
algorithm \verifyrcw (Algorithm \ref{procedure:extendable}) for \ppnps. 
The algorithm verifies if a given 
subgraph $G_s$ is a $k$-\rcw for 
a single node $v\in V_T$ \wrt 
$\M$ under $(k,b)$-disturbance. 
For a configuration $C$ with test set $V_T$, 
it suffices to invoke \verifyrcw 
for each $v\in V_T$. 
We show that such \PTIME solution 
exists with the following 
condition given in Lemma~\ref{lm-condition-iff}. Lemma~\ref{lm-condition-iff} itself does not require the notion of local budget; however, the \PTIME computation of $E^*_k$ (defined in Lemma~\ref{lm-condition-iff}) needs a local budget.

\begin{lemma}
\label{lm-condition-iff}
Given the configuration $C$ 
that specifies 
\ppnp $\M$, let $G_s$ be a verified \cw of $M(v,G)$ = $l$ 
for a node $v\in V_T$, and only $k$-disturbance is allowed, 
then $G_s$ is a $k$-\rcw of $M(v,G)$=$l$,  
if and only if $M(v, G\setminus E^*_k)$ = $l$, 
where $E^*_k$ = $\argmax_{E_k\subseteq G\setminus G_s, c\neq l} \pi_{E_k}\left(v\right)^{T}\left(Z_{\{:, c\}}-Z_{\{:, l\}}\right)$.
%
\end{lemma}

\begin{proofS}
Let $\widetilde{G}$ be the graph 
obtained by disturbing $E^*_k$ in 
$G\setminus G_s$.  
For the \textbf{If} condition, 
let $E^*_k$ be the node pairs to be disturbed in the corresponding 
$k$-disturbance in $G\setminus G_s$. As $E^*_k$ maximizes 
$\pi_{E_k}\left(v\right)^{T}\left(Z_{\{:, c\}}-Z_{\{:, l\}}\right)$, 
which indicates that it minimizes the gap between 
the likelihoods that {\em some} label $c\neq l$ being assigned to 
$v$, quantified by the worst-case margin 
$m^*_{l,c}$, for {\em every} label $c$ that 
is not $l$. 
Then, if $m^*_{l,c}>0$ for every $c\neq l$ under 
disturbance $E^*_k$, $M(v, \widetilde{G})$ 
will remain to be $l$. 
This means that no $k$-disturbance in $G\setminus G_s$ 
will change $v$'s label. 
Meanwhile, since disturbing $E_k^*$ alone is not sufficient enough to affect the outcome of $M$ on $v$, and as $G_s$ is a verified \cw for $M(v, G) = l$, so $M(v, G\setminus G_s) \neq l$, then $M(v, \widetilde{G}\setminus G_s) \neq l$ holds for any $k$-disturbance. Hence, $G_s$ is a $k$-\rcw of $M(v,G)$=$l$ by definition. 

The \textbf{Only If} condition can be shown by 
contradiction. Assuming there exists a set 
$E^*_k$ that makes  
$M(v, \widetilde{G})\neq l$, 
and $G_s$ is still a $k$-\rcw 
of $M(v, G)$ = $l$; since 
$\M$ is fixed and deterministic, 
$M(v, G_s)$ = $l$ is not 
affected by disturbing $E_k$ 
``outside'' of $G_s$; 
yet $M(v, \widetilde{G})\neq l$, 
hence $G_s$ is not a factual witness 
for $M(v, \widetilde{G})$ = $l$, 
and hence not a \cw for $M(v, \widetilde{G})$ = $l$, 
violating the third condition 
in the definition of $k$-\rcw 
for $G$. Hence $G_s$ is not a $k$-\rcw 
of $M(v,G)$ = $l$. 
This contradicts to 
that $G_s$ is a $k$-\rcw of $M(v,G)$=$l$. 
\end{proofS}

\begin{algorithm}[tb!]
    \renewcommand{\algorithmicrequire}{\textbf{Input:}}
    \renewcommand{\algorithmicensure}{\textbf{Output:}}
    \caption{Algorithm \verifyrcwp(single node)}  
    \begin{algorithmic}[1]
    \REQUIRE A configuration $C$ = $\{G, G_s, v, \M, k\}$, integer $b$; \\
    \ENSURE \kw{true} if $G_s$ is a $k$-\rcw, \kw{false} otherwise.\\
        
        \IF{$\verifyw(C)=$\kw{false} \textbf{or} $\verifycw(C)=$\kw{false}} 
        \STATE \textbf{return} \kw{false}; 
        \ENDIF

        \vspace{1ex}

        \STATE construct graph $G'$ = $G \setminus G_s$; 
        \STATE $\hat{A}^\prime$ = $A^\prime + I$, where $A^\prime$ is adjacency matrix of $G'$; 
        \STATE initialize edge set $E^0$ with arbitrary subset of $G'$;
        
        \vspace{1ex}
        
        \FOR{\textbf{each} $c\in L \setminus \{l\}$}
            \STATE $\mathbf{r}$ = ${Z}_{\{:,c\}}-{Z}_{\{:,l\}}$;
            \STATE $E^*$ = $\pri(\mathbf{r}, b, E^0)$;  \hspace{4ex}/*get $(k, b)$-disturbance*/
            \IF{$M(v, G\setminus E^*) \neq l$  \textbf{or} $|E^*| > k$}
            \STATE \textbf{return} \kw{false}; 
            \ENDIF 
        \ENDFOR
        \STATE \textbf{return} \kw{true};
    \end{algorithmic}

        \vspace{1.5ex}

    \textbf{Procedure}: $\pri(\mathbf{r}, b, E^0)$
    \begin{algorithmic}[1]
        \STATE initialize integer $i$=$0$; 
        \WHILE{$E^i\neq E^{i-1}$ \textbf{or} $i = 0$}
            \STATE update $\hat{A}^\prime$ by ``disturbing'' $E^i$ in $G$; /*flip entries*/
            \STATE update $X^L$ = $(I_N-\alpha\Hat{D}^{-1} \Hat{A}^\prime)^{-1}\cdot \mathbf{r}$;
            \STATE $s(u,u')$=$(1-2A^\prime_{uu'})$ $(X_{u'}-\frac{X_u-X_u'}{\alpha})$, $\forall (u, u') \in G'$;
            \STATE $E_{b_u}$ := $\text{top}_b\{(u, u') \in G' \mid s(u, u') > 0\}$, $\forall u \in G'$;
            \STATE $E_b^i$ = $\bigcup_{u \in G'}E_{b_u}$;
            \STATE $E^{i+1}$ = $(E^i \cup E^i_b) - (E^i \cap E^i_b)$, $i = i + 1$;
            \STATE \textbf{if} $M(v, G\setminus E^i) \neq l$ \textbf{then} break; 
        \ENDWHILE
        \STATE \textbf{return} $E^i$
    \end{algorithmic}
    \label{procedure:extendable}
\end{algorithm}

\eat{
\begin{algorithm}[tb!]
    \renewcommand{\algorithmicrequire}{\textbf{Input:}}
    \renewcommand{\algorithmicensure}{\textbf{Output:}}
    \caption{Algorithm \verifyrcwp(single node)}  
    \begin{algorithmic}[1]
\REQUIRE A configuration $C$ = $\{G, G_s, v, \M, k\}$, integer $b$; \\
        \ENSURE \kw{true} if $G_s$ is a 
        $k$-\rcw, \kw{false} otherwise.\\
         \IF{$\verifyw(C)$ = \kw{false}} 
        \STATE \textbf{return} \kw{false}; 
        \ENDIF
        \STATE construct graph $G'$ := $G\setminus G_s$; 
        \IF{$\verifycw(G',C)$ = \kw{false}}
        \STATE \textbf{return} \kw{false}; \ENDIF
        \STATE initializes bitmaps $A$ and $\hat{A}$ as adjacency matrix of $G'$; initializes integer $i$:=$0$; 
        \STATE initialize edge set $E^0_k$ with $k$ random edges in $G'$; a queue $Q$:=$\{V_T\}$; 
        \WHILE{$E^i_k\neq E^{i-1}_k$ \And $|E^i_k|\leq k$} 
        \STATE update $\hat{A}$ by ``disturbing'' $E^i$ in $G$; /*flip entries*/
        \STATE update $X^L = (I_N-\alpha\Hat{D}^{-1} \Hat{A})^{-1}\cdot (\textbf{Z}_{\{:,c\}}-\textbf{Z}_{\{:,l\}})$;
        \FOR{\textbf{each} $v\in Q$}
        \STATE choose top-$b$ adjacent edges $E^i_b$ of $v$ in $G'$ 
        with highest scores $s(v,v')$=(1-$2\cdot A_{(v,v')})$ $(X_{v'}-\frac{X_v-X_v'}{\alpha})$
        \STATE $E^i_k$ := $E^i_k \cup E^i_b$; $i$:=$i+1$;
        \STATE $Q$:=$Q\setminus\{v\}\cup\{v'\}$ ($v'\in E^i_b$)
        \IF{$M(v, G\setminus E^i_k)\neq l$} 
        \STATE \textbf{return} \kw{false}; 
        \ENDIF
        \ENDFOR
        \ENDWHILE    
        \IF{$M(v, G\setminus E^i_k)=l$} 
        \STATE \textbf{return} \textbf{true};
        \ENDIF
        \STATE \textbf{return} \textbf{false}
        
    \eat{
        \REQUIRE A configuration $C$ = $\{G, \emptyset, V_T, \M, k\}$, integer $b$ (local budget); \\
        \ENSURE A minimal $k$-\rcw $G_s$ of $V_T$ \wrt $\M$. \\
       \STATE $G_I$:=$\emptyset$; $G_s$:=$\emptyset$; set $E_k$:=$\emptyset$; integer $j$:=$1$;  
       \STATE $G_s$ := $\kw{Kruskal}(G,V_T)$; $G_I$:=$G_s$; 
        \WHILE{$G\setminus G_s\neq\emptyset$} 
            \STATE $G_s$:=$G_I$; $G_I$:= $\kw{Expand}(G_s, G, V_T)$; $j$:=$1$; 
            \FOR{$j = 1$ to $k$}
            \IF{$\verifyrcw(C, G_I, j)$ = \kw{true}}
                 \IF{j=$k$} 
                    \STATE \textbf{return} $G_I$; 
                 \ELSE  \STATE \textbf{continue}; 
                \ENDIF
            \ELSE \STATE $G_I$:=$G_s$; \textbf{break}; 
            \ENDIF
            \ENDFOR
        \ENDWHILE
       \STATE \textbf{return} $G_s$;  
       }
    \end{algorithmic}
    \label{procedure:extendable}
\end{algorithm}
}

Based on the above result, it suffices to 
show that there is a \PTIME solution to 
the optimization problem 
that computes $E^*_k$ (under $(k, b)$-disturbance), which 
maximizes 
$\pi_{E_k}\left(v\right)^{T}\left(Z_{\{:, c\}}-Z_{\{:, l\}}\right)$.
Following 
~\cite{bojchevski2019certifiable}, 
we present a greedy edge selection 
strategy to construct $E^*_k$ 
in \PTIME. 

\eat{
\eetitle{Algorithm}. 
Algorithm \ref{procedure:extendable} first invokes \verifyw and \verifycw to 
check if $G_s$ is a factual witness and a \cw (lines~1-2). 
If $G_s$ remains to be a \cw, it then 
follows an iterative 
``optimize-and-verify'' process, 
for each class $c \in L \setminus \{l\}$,
it computes a set of edges  
$E_k$ that can optimize 
$(Z_{\{:,c\}}-Z_{\{:,l\}})^T\pi_{E_k}\left(v\right)$ by Procedure \pri.
Each time an
$(k,b)$-disturbance with $|E_k| \leq k$
is constructed, it invokes \verifycw 
to verify if $G_s$ remains \cw by the updated configuration $C'$(line 11). 
If so, $G_s$ is not 
a $k$-\rcw as 
``disapproved'' by such an 
$(k,b)$-disturbance (Lemma~\ref{lm-k})
and returns ``flase''. 
Otherwise, the robustness of 
$G_s$ is asserted under the tests of $E_k^*$ from all possible labels. 

\eetitle{Procedure~\pri}.
It adapts Policy Iteration algorithm for PageRank Optimization~\cite{bojchevski2019certifiable} to select $E_k$ that mostly hurts $G_s$ as a \cw for $M(v, G) = l$. 
For each node $u \in G\setminus G_s$, it selects at most $b$ adjacent edges with the highest positive $s(u, u')$ and constructs $E_b^i$ (line 5-6). $s(u, u')$ measures the potential improvement on the PageRank score of test node $v$ by altering edge $(u, u')$ in $E_k^i$.
\pri greedily improves $E_k^i$ by removing or inserting the edges in $E_b^i$ to obtain $E_k^{i+1}$ and feeds it to next iteration (line 7). It terminates when no further improvement can be made on $E_k^*$ and returns the optimized $E_k^*$.

\stitle{Analysis}. 
The correctness is ensured by 
showing that the greedy selection process 
correctly solves the Personalized PageRank 
maximization problem,  
which meanwhile minimizes the worst-case margin (Lemma~\ref{lm-condition-iff}) following the analysis  in~\cite{bojchevski2019certifiable}, where a \PTIME solution 
is provided. 
For time cost, \verifyrcw iterates at most
$|G \setminus G_s|$ configurations in Procedure \pri and verify the optimized one for each class pair~\cite{hollanders2011policy},
and both \verifyw and \verifycw 
remain to be in \PTIME under 
$(k,b)$-disturbance (Lemma~\ref{lm-verifyw} and Lemma~\ref{lm-verifycw}). 
The total cost is thus in 
$O(L\cdot(|G \setminus G_s|^2\log|G \setminus G_s|+ LF(|E|+|V|F)))$ for representative 
\gnns. 
We leave the detailed analysis in~\cite{full}. 
}

\eetitle{Algorithm}. 
Algorithm \ref{procedure:extendable} first invokes \verifyw and \verifycw to 
check if $G_s$ is a factual witness and a \cw (lines~1-2). 
If $G_s$ remains to be a \cw, it then 
follows an iterative 
``optimize-and-verify'' process, 
which first computes a set of edges  
$E_k$ that can optimize 
$\pi_{E_k}\left(v\right)^{T}\left(Z_{\{:, c\}}-Z_{\{:, l\}}\right)$,
and then verifies if 
the $(k,b)$-disturbance 
changes the label.  
 
\eetitle{Procedure~\pri}.
It nontrivially optimizes the policy iteration procedure~\cite{bojchevski2019certifiable}, 
which goes through multiple 
rounds of top-$b$ edges selection to get $E^*$ that mostly hurts $G_s$ as a \cw for $M(v, G) = l$. 
In each round $i$, 
it calculates a score $s(u, u')$ for each node pair $(u, u') \in G'$, which measures the potential improvement on the PageRank score of test node $v$ by flipping $(u, u')$. 
Then, for each node $u \in G'$, it selects at most $b$ node pairs with the highest positive $s(u, u')$ and constructs $E_b^i$ (line 6-7).  
\pri greedily improves $E^i$ by removing or inserting the edges in $E_b^i$ to obtain $E^{i+1}$ and feeds it to the next iteration (line 8). Whenever a set of node pairs
$E^i$ 
is constructed, it checks if the label of 
$v$ changes by ``disturbing'' $G$ with $E^i$ (line 9). 
If so, $G_s$ is not 
a $k$-\rcw as 
``disapproved'' by $E^i$
and returns ``false''. 
Otherwise, 
$G_s$'s robustness will be asserted by testing all 
possible disturbances until it reaches a point where no further improvement can be made on $E^*$, at which point \pri terminates and returns the optimized $E^*$. 

\stitle{Remarks}. 
One notable aspect of Procedure \pri is that it only ensures local $b$-size disturbance on each node in $G'$ and it achieves global optimization without guaranteeing $k$-size disturbance. Consequently, in line 9 of Algorithm~\ref{procedure:extendable}, we only check the disturbance where $|E^i| \leq k$ 
to confirm that $G_s$, which is asserted as a $k$-\rcw, is indeed accurate. If Procedure \pri generates an $E^*$ that exceeds $k$ in size, we reject it as false (line 9, Algorithm~\ref{procedure:extendable}). This approach may accidentally reject some legitimate cases, but it is sufficient to meet the requirements of the subsequent generation algorithm (Section \ref{sec-rcwgen}) that invokes it.
The correctness is ensured by 
the greedy selection process (following a ``local budget'') that correctly solves the Personalized PageRank 
maximization problem,  
which meanwhile minimizes the worst-case margin (Lemma~\ref{lm-condition-iff}) following the analysis  in~\cite{bojchevski2019certifiable}, where a \PTIME solution 
is provided. 
For time cost, \verifyrcw iterates at most
$|G \setminus G_s|$ configurations in Procedure \pri~\cite{hollanders2011policy},
and both \verifyw and \verifycw 
remain to be in \PTIME 
(Lemma~\ref{lm-verifyw} and Lemma~\ref{lm-verifycw}). Setting $G' = G \setminus G_s$ and $d_m$ as the maximum node degree, the total cost is $O(L|G'|(d_m\log d_m+ LF(|E|+|V|F)))$ for representative 
\gnns. 


\section{RCW Generation Problem} 
\label{sec-rcwgen}

The generation problem, unlike verification,
is to compute a nontrivial $k$-\rcw if exists,  for a given 
configuration $C$ = $(G, \emptyset, V_T, M, k)$, 
or simply $(G, V_T, M, k)$ as the $k$-\rcw 
is to be computed.  
We can verify the following result. 

\begin{theorem}
\label{thm-rcw-gen} 
Given a configuration $C$ = $(G, V_T, M, k)$, 
the $k$-\rcw generation problem is in general co-\NP-hard. 
\end{theorem}

\begin{proofS}
It suffices to consider a ``single node''. 
A problem is co-\NP-hard if 
its complementary problem is 
\NP-complete. The decision problem of 
$k$-\rcw generation for  
a configuration $C$  = $(G, v, M, k)$ 
is to decide if there exists a subgraph 
$G_s$ such that for any $k$-disturbance, 
$G_s$ remains to be a \cw 
for $M(v, \widetilde{G})$=$l$, 
where $\widetilde{G}$ is obtained by 
disturbing $G$ with the $k$-disturbance. 
Its complementary problem 
is to verify if 
no subgraph $G_s$ in $G$ 
can be a $k$-\rcw, 
\ie for any subgraph $G_s$, 
there always exists 
some $k$-disturbance that 
``disapproves'' $G_s$ to be 
a $k$-\rcw. 
We make a case for \ppnps 
for the complementary problem. 
(1) There is an \kw{NP} algorithm 
that non-deterministically guesses 
a subgraph $G_s$ and a $k$-disturbance, 
and solves a verification problem 
given $G_s$ and $k$-disturbance 
in \PTIME with inference test. 
(2) The \NP-hardness follows from Theorem~\ref{lm-rwverify}, 
which shows that the verification problem of 
$k$-\rcw for \ppnps is already 
\NP-hard. Putting these together, 
Theorem~\ref{thm-rcw-gen} follows. 
\end{proofS}

One may consider a more feasible case 
when the verification problem is 
tractable. Following Lemma~\ref{lm-condition-iff}, 
we identify a tractable case for 
\ppnps under $(k,b)$-disturbances. 

\begin{lemma}
\label{thm-rcw-gen-tractable} 
Given a configuration $C$ = $(G, V_T, M, k, b)$ 
and only $(k,b)$-disturbances are allowed, 
the $k$-\rcw generation problem for \ppnps is 
in \PTIME. 
\end{lemma}

We 
next 
present a feasible algorithm 
to generate $k$-\rcw for \gnn $\M$. 
The algorithm ensures to output 
a $k$-\rcw in general cases; and 
in particular, ensures a \PTIME 
process for \ppnps 
under $(k,b)$-disturbance. 


\eat{
\begin{proofS}
We again consider a ``single node''. 
To see the upper bound, 
consider an \NP problem 
that guesses a subgraph 
$G_s$, and invokes 
the algorithm~\verifyrcw 
to verify if $G_s$ is a 
$k$-\rcw in \PTIME 
(Theorem~\ref{thm-feasible}). 
\end{proofS}
}

\eat{
\stitle{Generation Problem}. 
Given $G$, $V_T$, a \gnn $\M$ and a constant $k$, 
the (minimal) $k$-\rcw generation problem 
is to compute a (minimal) $k$-$\rcw$ 
$G_w$ of $V_T$ \wrt $\M$.  

It is often desirable to 
compute small explanations in terms of the 
number of edges to be inspected. 
We capture this with a minimality measure. 

\etitle{Minimal $k$-\rcw}. 
Given a configuration $C$ that 
specifies $G$, $V_T$, a \gnn $\M$ and a constant $k$, 
a subgraph $G_s$ of $G$ is a {\em minimal} 
$k$-\rcw, if for any subgraph $G_s'$ obtained by 
disturbing an edge from $G_s$, 
$G_s'$ is not a $k$-\rcw of 
$V_T$ \wrt $\M$. 

\stitle{Generation Problem}. 
Given $G$, $V_T$, a \gnn $\M$ and a constant $k$, 
the (minimal) $k$-\rcw generation problem 
is to compute a (minimal) $k$-$\rcw$ 
$G_w$ of $V_T$ \wrt $\M$.  

We show that the problem remains to be 
nontrivial, 
even when the corresponding verification 
problem is tractable. 

\begin{theorem}
\label{thm-rcw}
For a configuration $C$ that 
specifies $G$, $V_T$, a fixed, deterministic \ppnp 
$\M$ and fixed constants $k$ and $b$, 
the decision problem of minimal $k$-\rcw 
generation 
is \NP-complete, when only 
$(k,b)$-disturbance is allowed. 
\end{theorem}

\begin{proofS} 
The problem is in \NP: 
An \NP algorithm  
first guesses a subgraph 
$G_s$, and invokes algorithm~\verifyrcw to 
verify if $G_s$ is a $k$-\rcw under 
$(k,b)$-disturbance in \PTIME (Theorem~\ref{lm-rwverify-tractable}). 
The \NP-hardness of the problem 
can be verified by specifying 
$\M$ as a fixed \ppnp, and 
by performing a reduction 
from the Max-PageRank problem, 
a known \NP-hard problem 
that finds $k$ edges 
to maximize Personalized 
PageRank scores of a 
target node. \textcolor{red}{[AK: Can we elaborate more how that reduces to ''minimal'' k-RCW generation?]}
\end{proofS}

\textcolor{red}{[AK: What about the hardness of any k-RCW generation (not necessarily ''minimal'')? Is it still NP-complete?]}

\textcolor{red}{[AK: What about the existance of a k-RCW given any configuration? Why should it always exist?]}

Despite the hardness, 
we next introduce feasible 
algorithms for 
the generation problem. 
}

\begin{algorithm}[tb!]
    \renewcommand{\algorithmicrequire}{\textbf{Input:}}
    \renewcommand{\algorithmicensure}{\textbf{Output:}}
    \caption{Algorithm \mrcwgen} 
    \begin{algorithmic}[1]
        \REQUIRE A configuration $C$ = $\{G, V_T, M, k\}$, integer $b$ (local budget); \\
        \ENSURE A $k$-\rcw $G_s$ of $V_T$ \wrt $\M$. \\
        \STATE 
        initialize $G_s$=$\V_T$; set $E_k$=$\emptyset$;  
        queue $Q$ = $V_T$; 
         \WHILE{$G\setminus G_s\neq\emptyset$}  %
         \STATE update $C$=$(G,Q,M,k)$; 
         \IF{$\verifyrcw(C,G_s)$=\kw{false}}
         \STATE \textbf{return} $G$; 
         \ENDIF
         \WHILE{$Q\neq\emptyset$}
            \STATE $Q$=$\{v|v~\text{is not in $G_s$}, v\in V_T\}$; 
            \STATE node $v$=$Q.\kw{dequeue}()$; 
            \STATE $G_s$=\kw{Expand}$(v,G_s,C)$;
            \IF{$\verifyrcw(C,G_s)$=\kw{false}}
                \STATE \textbf{return} $G$; 
            \ENDIF
         \ENDWHILE
         \STATE \textbf{return} $G_s$; 
         \ENDWHILE
          \STATE \textbf{return} $G_s$; 
        \eat{\STATE 
       $G_I$:=$\emptyset$; $G_s$:=$\emptyset$; set 
       set $V_u$:=$V_T$; integer $j$:=$1$;  
       \STATE $G_s$:= $V_T$;  $G_I$:=$G_s$; 
        \WHILE{$G\setminus G_s\neq\emptyset$}
            \STATE $G_s$=$G_I$; $G_I$= $\kw{Expand}(G_s, G, V_T)$; $j$=$1$; 
            \FOR{$j = 1$ to $k$}
            \IF{$\verifyrcw(C, G_I, j)$ = \kw{true}}
                 \IF{j=$k$} 
                    \STATE \textbf{return} $G_I$; 
                 \ELSE  \STATE \textbf{continue}; 
                \ENDIF
            \ELSE \STATE $G_I$=$G_s$; \textbf{break}; 
            \ENDIF
            \ENDFOR
        \ENDWHILE
       \STATE \textbf{return} $G_s$; 
       }
    \end{algorithmic}
   \vspace{2ex}
        \textbf{Procedure}: $\kw{Expand}(v,G_s,C)$\\
        \vspace{-2.3ex}
        \begin{algorithmic}[1] 
        \STATE set $N_k$= $\{u|u \in G\setminus G_s\}$; 
        \STATE queue $v.Q_k$=$\emptyset$; set $E_k$=$\emptyset$;
        \WHILE{there is an unvisited pair $(u,u')$ in $ G\setminus G_s$}
        \STATE $v.Q_k$=${(u,u')}$;
        \STATE update worst-case margin $m^*_{l,c}(v)$; 
        \STATE update $E_k$ as all pairs $(u,u')$ that 
        maximize the worst-case margin; 
        \STATE augment $G_s$ with $E_k$; 
        \ENDWHILE
        \STATE \textbf{return} $G_s$; 
        \end{algorithmic}
    \label{procedure:rcwgen}
\end{algorithm}

\section{Generating Robust Witnesses}
\label{sec-approx} 

We present our main result below. 

\begin{theorem}
\label{thm-feasible}
Given a configuration $C$ 
that specifics $G$, $\M$, 
$V_T$, and $k$, 
there is an algorithm that 
generates a $k$-\rcw in 
$O((N+|G|)(L|E|F+L|V|F^2))$ time, where $N$ is the 
total number of verified $k$-disturbances. 
Specifically for \ppnps,
given $G' = G\setminus G_s$ and $d_m$ as maximum degree,
it is in $O(L|G'||V_T|\cdot(d_m\log d_m+ LF(|E|+|V|F) + k))$ time. 
\end{theorem}

We present an algorithm as a 
constructive proof. 
Following Lemma~\ref{lm-condition-iff}, 
it suffices to ensure that the algorithm 
verifies the sufficient and necessary condition. 
To this end, it 
processes a set of test nodes $V_T$ ``one node at a time'', 
and iteratively grows $G_s$ with an   
``expand-verify'' strategy for each node: 
\eat{reorder unvisited test nodes in 
$V_T$ to find the node that $G_s$ is likely to continue 
to be a $k$-\rcw, \ie its label is {\em not} likely to be 
changed given current $G_s$ and $k$-disturbances in $G\setminus G_s$ (see Procedure~\kw{PrioritizeQ}); 
}
(1) It {\em approximates} the optimal set of 
node pairs $E^*_k$ that can minimize the 
worst-case margin $m^*_{l,c}$ of $v$ (Eq.~\ref{eq-margin}), 
\ie most likely to change its label if ``flipped''
\footnote{For \ppnps, maximizing Personalized PageRank score as in Lemma~\ref{lm-condition-iff}.} (see Procedure~\kw{Expand}); 
and (2) augment $G_s$ to contain  
$E^*_k$, such that it ``secures'' 
$G$ by preventing disturbances from 
$E^*_k$. As such, $M(v, G\setminus G_s)$ 
is most likely\footnote{Here the optimality of $E^*_k$ is not guaranteed by the procedure, and needs to be 
verified, for general cases.} 
to be $l$, and 
$G_s$ will remain to be a $k$-\rcw 
for $M(v,G\setminus G_s)$ determined 
by the next round of verification. 

\stitle{Algorithm}. 
The algorithm, denoted as \mrcwgen 
and illustrated in Algorithm~\ref{procedure:rcwgen}, 
outputs a non-trivial $k$-\rcw $G_s$ if exists; 
and terminates with the trivial $k$-\rcw $G$ 
by default. 
It uses (i) a queue $Q$ to 
coordinate the processing 
of the test nodes $V_T$, (ii) for each test node $v\in V_T$, 
a queue of node pairs $v.Q_k$, to track the $k$-disturbance 
that most likely 
to change its own label. 
(1) It initializes $G_s$ as a trivial \cw set $V_T$ (line 1), 
and grows $G_s$ up to the trivial $k$-\rcw $G$ to 
guarantee the termination (lines 2-12). 
(2) The expansion of $G_s$ is 
controlled by a 
node-by-node 
process (lines~7-12). Given a verified $k$-\rcw $G_s$
for a fraction of $V_T$ (lines~3-5); 
it invokes procedure
~\kw{Expand} to process $v$ and to grow $G_s$. 
(3) The process continues until 
either a nontrivial $k$-\rcw $G_s$ for 
$V_T$ is identified and is returned (line 13), or no such 
structure can be found, and trivial 
explanations are returned (lines 5 and 11). 

\eetitle{Procedure~\kw{Expand}}. 
The procedure~\kw{Expand} iteratively explores an 
unvisited node pair $(u,u')$ in $G\setminus G_s$ 
and iteratively estimates if the worst-case margin 
is affected by including $(u,u')$ in a 
$k$-disturbance. If $m^*_{l,c}$ is less than $0$, 
a $k$-disturbance is enlisted to $v.Q_k$ to 
augment $G_s$ for further verification. 

In particular for \ppnps under $(k,b)$-disturbance, ~\kw{Expand} 
implements the greedy policy iteration process 
in the procedure~\verifyrcwp (which replaces lines 5-6 with an 
implementation that simulates lines 7-10 of \verifyrcwp), to expand $G_s$ with the 
optimal edge set $E^*_k$

\begin{figure}[tb!]
\centering
\centerline{\includegraphics[width =0.46\textwidth]{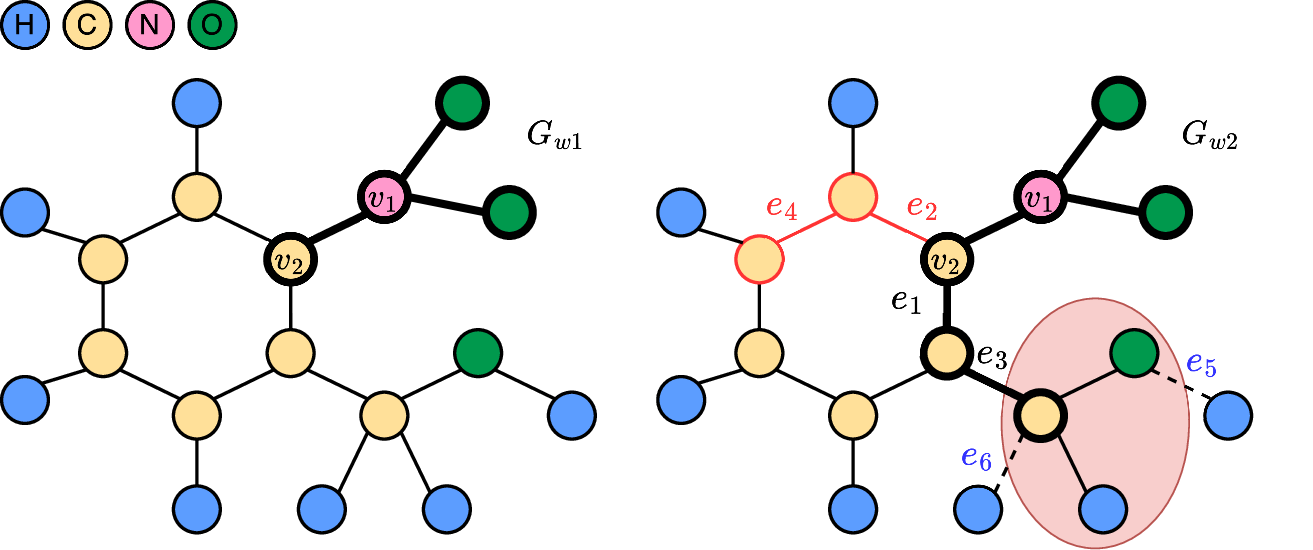}}
\vspace{-1.5ex}
\caption{Generating robust witness for $G_1$ with expansion and verification. }
\label{fig:example}
\vspace{-2ex}
\end{figure}

\begin{example}
 Figure \ref{fig:example} illustrates 
 the expansion and verification process. 
 Consider the mutagenic molecule graph $G_1$ 
 with two test nodes $V_T$ = $\{v_1, v_2\}$, and set $k$=2, $b$=1. Both nodes $v_1$ and $v_2$  
 are labeled `mutagenic' by an \ppnp $\M$. 
A currently verified $2$-\rcw $G_{w1}$ is identified 
for node $v_1$ (with bold edge), which contains 
$v_1$ and its three direct neighbors. 

\sstab
(1) Given $G_{w1}$, we start with $v_2$. On the right figure, \mrcwgen invokes \kw{Expand} to estimate, from the "changeable" area $(G_1\setminus G_{w1})$, those node pairs that are most likely to maximize the worst-case margin of $v_2$ from `mutagenic' to `nonmutagenic' if disturbed. Assume that we have two possible $E_{expand}^*$ for 
augmentation, which lead to two possible $G_s$ to be verified after \expand: $G_{w1}\cup \{e_1, e_3\}$ and $G_{w1}\cup \{e_2, e_4\}$. 
\sstab
(2) In the verification phase, \verifyrcwp invokes \pri to generate disturbance $E_{disturb}^*$, i.e., $\{e_5, e_6\}$. The red bold area indicated in Figure \ref{fig:example} is an aldehyde structure and shows the `mutagenic' property under the disturbance. The $G_{w1}\cup \{e_2, e_4\}$ fails the verification since the augmented part is not associated with any `mutagenic' structure, while $G_{w1}\cup \{e_1, e_3\}$ reaches to the carbon node in the aldehyde structure, which is responsible for `mutagenic', ensuring the inclusion of `mutagenic' structure and explains the label of $v_2$ under disturbance. Finally, the bold structure on the right graph, denoted as $G_{w2}$, is the \rcw that it finds. 
\end{example}

\eat{
In the next iteration, we add either edge 3 or edge 4, respectively, and remove edge 6. Now $G_{w1}$ with additional edges 2 and 4 is no longer a \cw since the aldehyde structure in the red area is toxic. Meanwhile, $G_{w1}$ with additional edges 1 and 3 remains a \cw, we denote it as $G_{w2}$. Finally, we generate an \rcw $G_{w2}$ w.r.t. given $k$ and $b$ constants. 
}

\vspace{-.5ex}
\stitle{Analysis}. 
Algorithm~\mrcwgen always terminates with 
either a non-trivial $k$-\rcw, or trivial cases. 
For the time cost, (1) it takes a total verification cost 
in $O(N\times I)$ 
(line~4), where $N$ is the total number of $k$-disturbances 
being verified, and $I$ is the \gnn-specific 
single inference cost, which is typically bounded in  
$O(L|E|F+L|V|F^2)$ for representative \gnns
(\eg \gcns, \gats, \gsage, \ppnps)~\cite{chen2020scalable,zhou2021accelerating}. 
Procedure~\kw{Expand} takes
$O(|G|I)$ time, as $G_s$ grows at most $|V|+|E|$ times\footnote{Note that although $G_s$ considers node pairs in $|V|\times|V|$, it only grows to size $|G|$ and does not add node pairs already included.}. 
The total cost is thus in $O((N+|G|)L|E|F+L|V|F^2)$.  

Specifically for \ppnps, the algorithm~\mrcwgen 
invokes procedure~\verifyrcwp to 
verify $k$-\rcw, set $G' = |G \setminus G_s|$, 
$d_m$ as the maximum degree, 
as analyzed in Algorithm~\ref{procedure:extendable}, it takes $O(L|G'|\cdot(d_m\log d_m+ LF(|E|+|V|F)))$ time to process each node $v \in V_T$.
The optimal node pairs $E^*_k$ 
for a node 
$v$ only need to be computed once in 
\verifyrcwp (line~8) 
or by \kw{Expand} (line~6) via a 
deterministic process, hence 
the procedure \kw{Expand} {\em does not} incur 
additional verification, and has cost 
$O(k|V_T|)$. Thus the total cost is  
$O(L|G'||V_T|\cdot(d_m\log d_m+ LF(|E|+|V|F) + k))$. 
Here $L$ is usually a small constant, 
and $|V_T|$ and $d_m$ are 
relatively much smaller compared with 
$|V|$ and $|E|$, respectively.

\vspace{.5ex}
Given the above analysis, Lemma~\ref{thm-rcw-gen-tractable} and Theorem~\ref{thm-feasible} follow. 

\eat{
\stitle{Minimality of $k$-\rcws}. It is often desirable to compute small explanations in terms of the number of edges to be inspected. 
We capture this with a minimality measure. 
A subgraph $G_s$ of $G$ is a {\em minimal} 
$k$-\rcw, if for any subgraph $G_s'$ obtained by 
disturbing an edge from $G_s$, 
$G_s'$ is not a $k$-\rcw of 
$V_T$ \wrt $\M$. We show that with minor revisions, algorithm~\mrcwgen can 
readily generates minimal \rcws without 
incurring additional overhead. Specifically, 
\mrcwgen uses a revised \expand and 
\verifyrcw as follows. 
(1) Instead of a batch estimation 
and expansion of multiple node pairs, 
\expand is enforced to perform a ``best-at-a-time'' 
growth in the procedure~\expand, such that 
the explanation is verified each time an edge 
is inserted to $G_s$.  
(2) Upon the expanded $G_s$, 
\verifyrcw performs at most 
$|G_s|$ round of additional inference tests (by invoking $M$), 
and in each round, removes an edge $e$ 
in $G_s$ to test if the rest remains a 
$k$-\rcw. Note that it suffices 
for \verifyrcw 
to {\em incrementally} check if $G_s\setminus\{e\}$ remains to be a 
$k$-\rcw, as $G_s$ is a verified $k$-\rcw for $M(v, G)$=$l$, 
and \verifyrcw only considers additional $k$-disturbances 
that must include $e$, which ``hot starts'' 
from $k-1$-\rcw verification in~\verifyrcw. 

Due to limited space, we present 
the details in~\cite{full}. 
}

\eat{
\subsection{Verification 
with Vulnerability Constraints} 

As a part of 
the constructive proof for 
Theorem~\ref{thm-feasible}, we revisit 
verification problem. 

\begin{lemma}
\label{lm-cwverify-local}
Given a graph $G$, a test node $v$,  
a fixed deterministic \gnn $M$, 
a set of fixed vulnerable edges $\hat{E}\subseteq E$ and a local 
budget $k$, 
it is in \PTIME to check if 
a subgraph $G_s$ of $G$ is a $k$-\rcw \wrt $\hat{E}$ and local budget 
$k$. 
\end{lemma}

We introduce such an algorithm. 

\subsection{Computing Minimal 
Robust Witnesses}

We next introduce our algorithm to 
compute minimal $k$-\rcws. 
}

\eat{
\begin{algorithm}[tb!]
    \renewcommand{\algorithmicrequire}{\textbf{Input:}}
    \renewcommand{\algorithmicensure}{\textbf{Output:}}
    \caption{Algorithm \mrcwgen} 
    \begin{algorithmic}[1]
        \REQUIRE A configuration $C$ = $\{G, \emptyset, V_T, \M, k\}$, integer $b$ (local budget); \\
        \ENSURE A minimal $k$-\rcw $G_s$ of $V_T$ \wrt $\M$. \\
       \STATE $G_I$:=$\emptyset$; $G_s$:=$\emptyset$; set $E_k$:=$\emptyset$; set $V_u$:=$V_T$; integer $j$:=$1$;  
       \STATE $G_s$:= $V_T$;  $G_I$:=$G_s$; 
        \WHILE{$G\setminus G_s\neq\emptyset$}
            \STATE $G_s$:=$G_I$; $G_I$:= $\kw{Expand}(G_s, G, V_T)$; $j$:=$1$; 
            \FOR{$j = 1$ to $k$}
            \IF{$\verifyrcw(C, G_I, j)$ = \kw{true}}
                 \IF{j=$k$} 
                    \STATE \textbf{return} $G_I$; 
                 \ELSE  \STATE \textbf{continue}; 
                \ENDIF
            \ELSE \STATE $G_I$:=$G_s$; \textbf{break}; 
            \ENDIF
            \ENDFOR
        \ENDWHILE
       \STATE \textbf{return} $G_s$;     
    \end{algorithmic}
    \label{procedure:extendable}
\end{algorithm}

\begin{algorithm}[tb!]
    \renewcommand{\algorithmicrequire}{\textbf{Input:}}
    \renewcommand{\algorithmicensure}{\textbf{Output:}}
    \caption{:\mrcwgen} 
    \begin{algorithmic}[1]
        \REQUIRE A graph $G$, the test node set $V_T$ and $k$. 
        \ENSURE A Minimal $k$-RCW. \\
        /*{\em Initialize Jacobian Matrix}*/
        \STATE Jacobian Matrix $JM = |V \times V_T|$
        \FOR{$e \in JM$}
        \STATE $e = \frac{\partial v\in V_T}{\partial V}$
        \ENDFOR        
        /*{\em Non-touchable core area}*/
        \STATE Generate a spanning tree $ST=Kruskal(G, V_T)$. $(|V_{ST}| = |V_T|, |E_{ST}| \subset |E_G|)$\\
        /*{\em Level-wise Search}*/
        \STATE $G_I = ST$. 
        \WHILE{$|G_I| \leq k$}
            \STATE Update $JM$ with node label propagation (1-hop). 
            \STATE $e = argmax(Sensitivity(JM))$. 
            \IF{Exist attack graph $G_A = Validator(G_I)$}
            \STATE $G_I = G_I \backslash G_A$. 
            \STATE $G_I = G_I \cup e$. 
            \ELSE
            \STATE $kRCW = G_I$
            \STATE \textbf{break}
            \ENDIF
        \ENDWHILE
        \RETURN $kRCW$;    
    \end{algorithmic}
    \label{procedure:extendable}
\end{algorithm}
}

\eat{
We summarize our main results in Table~\ref{tab:results}.

\begin{table}[tb!]
\begin{small}
\vspace{-3mm}
\begin{tabular}{c|c|c|c|c}
\textbf{Setting} & \textbf{Hardness} & \textbf{Algorithm} & \textbf{Guarantee} & \textbf{Time} \\ 
\hline
Witness Verification  & PTIME &  &  & \\ \hline  
\cw Verification & PTIME & & & \\ \hline 
\rcw Verification & NP-hard & & &  \\ \hline
minimal \rcw & CoNP-hard & & & \\ \hline 
\hline
\end{tabular}
\end{small}
\caption{Main results, Algorithms and 
Guarantees \tbf}  
\label{tab:results}
\vspace{-2ex}
\end{table}
}

\section{Parallel Witness Generation}
\label{sec-parallel}

When the graph is large, the verification 
of $k$-\rcw can be expensive and is a major 
bottleneck. 
We next present a parallel algorithm, 
denoted as \prcwgen and is illustrated 
as Algorithm~\ref{alg:parallel}. 
Our idea is to parallelize the 
verification into a smaller manageable 
tasks \textless$G_i, G_s, j$\textgreater, each of 
which checks ``if there is a disturbance from $G_i$ alone that disproves $G_s$ to be a  
$j$-\rcw of $M(v,G)$=$l$'', and performs parallelized verification for each 
fragment $G_i$ of $G$.   

Given Lemma~\ref{lm-k}, we observe the following result. 

\begin{lemma}
\label{lm-parallel}
Given a subgraph $G_i$ of $G\setminus G_s$, 
If there is a $j$-disturbance $E^i_j$ 
in $G_i\setminus G_s$, and applying $E^i_j$ 
to $G$ changes the 
result $M(v,G)$ = $l$, then 
$G_s$ is not a $k$-\rcw for 
$M(v,G)$=$l$, for any $k\geq j$. 
\end{lemma}

\begin{figure}
\centering
\begin{algorithm}[H]
    \renewcommand{\algorithmicrequire}{\textbf{Input:}}
    \renewcommand{\algorithmicensure}{\textbf{Output:}}
    \caption{Algorithm~\prcwgen} 
    \begin{algorithmic}[1]
        \REQUIRE Graph $G = \{G_1, G_2, \dots, G_n\}$, test nodes $V_T$, and $k$. 
        \ENSURE A $k$-\rcw $G_s$ of $V_T$ \wrt $\M$. \\
        /*{\em At Coordinator $s_0$}*/
        \STATE initialize $G_s$=$\V_T$; set $E_k$=$\emptyset$; 
        queue $Q$ = $V_T$; 
        \STATE compute bitmap $B$ of adjacency matrix of $G$ 
        \WHILE{$Q\neq\emptyset$} 
        \STATE $Q$=$\{v|v~\text{is not in $G_s$}, v\in V_T\}$; 
        \STATE node $v$=$Q.\kw{dequeue}()$;   
        \FOR{$j = 1$ to $k$} 
        
            \STATE $G^i_{s_j}$ = $\paraexpand(C,G_i,G_s,j,B_j)$; 
            /*{\em parallel expansion}*/
             
            \STATE $\paraverifyrcw(C,G^i_{s_j},G_s,j,B_j)$;/*{\em Parallel verification}*/
            \STATE synchronize $B$; 
            
            \STATE $G_{s_j}$= $\bigcup_{i=1}^{n}G^i_{s_j}$;
            \IF{$\verifyrcw(C,G_{s_j},B_j)$=\kw{false}}
                \STATE update $B_j$; \textbf{continue};  
            \ENDIF
        \ENDFOR
        \ENDWHILE
        \STATE \textbf{return} $G_s$;   
        \vspace{1ex}\\
        \eat{
        /*{\em Parallel execution at worker $s_j$}*/
        \WHILE{$Q_t$ is not empty}
                \STATE $G_j$ = $Q_t$.pop(), $G^i_{s_j} = G^{i-1}_{s_j}$.
                \STATE Identify sensitive edges and update $G^i_{s_j}$.
                \IF{\verifyrcw($G$, $G^i_{s_j}$, $V_T$, $i$, $B^i$)}
                    \STATE $Q$.add($G_j$).
                \ENDIF
            \ENDWHILE
}
    \end{algorithmic}
\end{algorithm}
\vspace{-3ex}
\label{alg:parallel}
\end{figure}

\eat{
\begin{figure}
\centering
\begin{algorithm}[H]
    \renewcommand{\algorithmicrequire}{\textbf{Input:}}
    \renewcommand{\algorithmicensure}{\textbf{Output:}}
    \caption{Parallel Witness Generation} 
    \begin{algorithmic}[1]
        \REQUIRE Graph $G = (G_1, G_2, \dots, G_n)$, test nodes $V_T$, and $k$. 
        \ENSURE A minimal $k$-\rcw $G_s$ of $V_T$ \wrt $\M$. \\
        \STATE initialize a bitmap array $D$ with size $k\times|E|$.
        \vspace{1ex}
        \FOR{\textbf{each} $V_{T_n}$ \textbf{in} $Q$}
            \STATE $R^0[V_{T_n}]$ = GenerateSpanningTree($G$, $V_{T_n}$).
        \ENDFOR
        \vspace{1ex}
        \FOR{$i = 1$ to $k$}
            \STATE \textbf{set} $Q_t$ = $Q$, $Q=\varnothing$.
            \WHILE{$Q_t$ is not empty}
                \STATE $V_{T_n}$ = $Q_t$.dequeue(), $G_S = R^{i-1}[V_{T_n}]$.
                \STATE Update $JM_n$ with node label propagation (1-hop). 
                \STATE Identify sensitive edges and update $G_S$ with $JM_n$.
                \IF{\verifyrcw($G$, $G_S$, $V_{T_n}$, $i$, $D^i$)}
                    \STATE $R^i[V_{T_n}] = G_S$, $Q$.enqueue($V_{T_n}$).
                \ENDIF
            \ENDWHILE
            \STATE $i$-\rcw $= \bigcup_{j=1}^{n}R^i[V_{T_j}]$
            \IF{not \verifyrcw($G$, $i$-\rcw, $V_T$, $i$, $D^i$)}
            \STATE identify and remove conflicts from $i$-\rcw and $R$.
            \ENDIF
        \ENDFOR
        \vspace{1ex}
        \RETURN $k$-\rcw;    
    \end{algorithmic}
\end{algorithm}
\label{alg:parallel}
\end{figure}
}

\stitle{Partition}. The algorithm \prcwgen 
works with a coordinator site $S_0$ and 
$n$ workers $\{S_1, \ldots S_n\}$. 
Given a configuration $C$ = $(G,V_T,M,k)$, 
it ensures an ``inference preserving partition'' 
such that local verification and 
inference can be conducted. 
(1) The \gnn model $\M$ is duplicated 
at each site. 
(2) The graph $G$ is fragmented 
into $n$ partitions $\{G_1, \ldots G_n\}$ through edge-cut based partition
where each worker $S_i$
processes one fragment $G_i$. 
For each ``border node'' $v$ that resides 
in fragment $G_i$, 
its $k$-hop neighbors are duplicated 
in $G_i$, 
to ensure that no data exchange is needed 
for parallel verification. 
(3) Each row of the adjacency matrix is converted into a bitmap, a sequence of bits where each bit represents the presence (`1') or absence (`0') of an edge from the vertex corresponding to that row to another vertex.
All local fragments share 
a compressed bitmap 
encoding $B$ of the adjacency matrix of $G$, 
the node features $X^0$, 
and the common $V_T$, such that 
$M(v,G)$ can be inferred locally 
without expensive data communication. 

\stitle{Algorithm}. Algorithm \prcwgen is 
illustrated as Algorithm~\ref{alg:parallel}. 
It starts with the initialization 
similar to its sequential counterpart~\mrcwgen. 
In addition, it uses one bitmap for each worker to record verified $k$-disturbances
and a global bitmap
$B$ to synchronize all $k$-disturbances verified from each worker to avoid redundant verification. 
It mainly parallelizes the 
computation of two procedures: \kw{Expand}, 
with procedure \paraexpand (line~8), 
and \verifyrcw, with procedure~\paraverifyrcw (line~9) 
(not shown). 
It assembles an expanded 
global subgraph $G_{s_j}$ as the union 
of all verified local counterparts (line~11). 
As this does not necessarily ensure 
all disturbances are explored, it continues 
to complete a coordinator side verification, yet does 
not repeat the verified local ones (lines~12-13) 
using the global bitmap $B$.

\vspace{.5ex}
The procedure \paraexpand 
grows the current verified $j-1$-\rcw $G_{s_{j-1}}$ 
at each site in parallel, 
and sends the locally expanded 
$G^i_{s_j}$ to the coordinator with locally 
expanded node pairs that can optimize the 
worst case margin $m^*_{l,c}(v)$.  
Similarly, the procedure \paraverifyrcw 
examines and only sends $G^i_{s_j}$ that 
remains to be a \cw, 
and updates the bitmap $B$ 
with the verified local $k$-disturbance. 

\stitle{Analysis}. 
The algorithm \prcwgen correctly 
parallelizes the computation of 
its sequential counterpart \mrcwgen, 
as guarded by the coordinator side verification, 
and reduces unnecessary data shipment 
and verification, 
given Lemma~\ref{lm-parallel}. 
For the parallel cost, 
there are at most $k|V_T|$ rounds 
of parallel computation. 
The total time cost (including communication cost) is 
$O(\frac{|G|+|B|)(L|E|F+L|V|F^2)}{n} 
+ I_0\cdot L|E|F+L|V|F^2)$. 
Here $I_0$ refers to the number of $k$-disturbances 
verified at the coordinator. 
Algorithm \prcwgen~ reduces 
unnecessary and redundant verification. 
This analysis verifies that \prcwgen 
scales well as more 
processors are used. 

\eat{
For the total parallel time cost, 
there are at most $k|V_T|$ rounds 
of parallel computation, and 
$2k|V_T|$ rounds of 
communications between 
the coordinator 
and the workers. In 
each communication, 
(1) \paraexpand incurs on average 
$\frac{|G^i_{s_j}|}{n}$ parallel time cost 
and a data shipment of $\frac{|B_j|+|G^i_{s_j}|}{n}$. 
and (2) \paraverifyrcw incurs a parallel time cost 
in $O(\frac{|B_j|(L|E|F+L|V|F^2)}{n})$ 
time, with a data shipment of $\frac{|B_j|}{n}$. 
Thus the total time cost (including communication cost) is in 
$O(\frac{|G|+|B|)(L|E|F+L|V|F^2)}{n} + I_0\cdot L|E|F+L|V|F^2)$. 
Here $I_0$ refers to the number of $k$-disturbances 
verified at the coordinator site, 
which is disjoint with any verified counterparts from 
any local site; that is, \prcwgen~{\em avoids} 
unnecessary redundant verification. 
This analysis verifies that \prcwgen 
scales well as more 
processors are used. 
}
\section{Experimental Study} 
\label{sec:exp}

We experimentally evaluate our algorithms with three real-world and one synthetic datasets. We evaluate the following: 
(\textbf{RQ1}): (a) the quality of the $k$-\rcws generated by \rgexp in terms of robustness and fidelity measures; 
(b) the impact of critical factors, including test node size 
$|V_T|$  and degree of disturbances ($k$), to the quality; 
(\textbf{RQ2}): (a) the efficiency of \rgexp and 
\prgexp for large graphs, (b) the impact of 
$|V_T|$ and $k$ to efficiency,  
and  
(\textbf{RQ3}): The scalability of \prgexp. 
We also perform two case study analyses to showcase real applications of \rgexp. The codes and datasets 
are made available at~\cite{code}. 

\subsection{Experiment settings} 
\label{sec:exp:settings}
\vspace{-1ex}
\stitle{Datasets}. Our datasets are summarized below (Table~\ref{tab:dataset}):

\sstab
(1) \underline{\house}~\cite{ying2019gnnexplainer}, a synthetic dataset, uses a Barab\'{a}si-Albert graph as the base with a house motif. Each node has 5 neighbors on average, with motifs labeled 1, 2, and 3 as `roof', `middle', and `ground'. The remaining nodes are labeled 0.

\sstab
(2) \underline{\ppi}~\cite{ppi}, a protein-protein interaction dataset, contains human proteins (nodes) and their interactions (edges). Node features include motif gene sets and immunological signatures, labeled by gene ontology sets reflecting functional properties.
\eat{The \ppi dataset is a protein-protein interaction dataset, each node refers to human proteins and each edge refers to interactions between proteins. The node features are motif gene sets and immunological signatures, and the node labels are gene ontology sets, which indicate the functional properties.}

\sstab
(3) \underline{\citeseer}~\cite{citeseer} is a citation dataset for computer science. Each node is a publication with edges as citation relations. There are 6 classes: Agents, AI, DB, IR, ML, and HCI. ``Agents'' specifically pertains to papers related to the field of autonomous agents and multi-agent systems. Each node features a binary vector for keyword occurrences.

\sstab
(4) \underline{\reddit}~\cite{reddit} is a large-scale social network dataset with hundreds of millions of edges and a set of nodes (posts) with features and node classes. Node features are word vectors, and node labels are communities in which the post belongs. 

For injecting $k$-disturbances, we adopt 
a strategy that mainly removes existing edges to 
capture cases that establish new 
links in real networks may be expensive~\footnote{Other disturbance strategy exists such as insertion-only or random attacks; the best choice is application-specific.}. 


\begin{table}[]
\caption{Statistics of Datasets}
\label{tab:dataset}
\resizebox{\columnwidth}{!}{%
\begin{tabular}{c|cccc}
\hline
Dataset & \# nodes & \# edges & \# node features & \# class labels \\ \hline
\house & 300 & 1500 & - & 4 \\
\ppi & 2,245 & 61,318 & 50 & 121 \\
\citeseer & 3,327 & 9,104 & 3,703 & 6 \\
\reddit & 232,965 & 114,615,892 & 602 & 41 \\ \hline
\end{tabular}%
}
\end{table}


\stitle{GNN Classifiers.} 
Following the evaluation of \gnn explanation 
\cite{yuan2021explainability,ying2019gnnexplainer,zhang2022gstarx,huang2023global}, we adopt standard message-passing graph convolutional network (\gcns) configured with $3$ convolution layers, each featuring an embedding dimension of 128. We remark that our solutions are model-agnostic and generalize to \gnn specifications. 

\eat{In line with recent works \cite{yuan2021explainability,ying2019gnnexplainer,zhang2022gstarx,huang2023global}, we employ a classic message-passing \gnn, namely a graph convolutional network (GCN) with three convolution layers, each having an embedding dimension of 128. To facilitate classification, the GCN model is enhanced with a max pooling layer and a fully connected layer. For datasets without node features, we assign each node a default feature. During training, we utilize the Adam optimizer \cite{kingma2014adam} with a learning rate of 0.001 for 2000 epochs. The datasets are split into 80\% for training, 10\% for validation, and 10\% for testing. The explanations are generated based on the classification results of the test set. Recall that our proposed solutions are model-agnostic, making them adaptable to any \gnn employing message-passing.}

\stitle{GNN Explainers}. We compare \rgexp against two recent 
GNN explainers. We use the original codes and authors' configurations for a fair comparison.
(1) \underline{\cfexplainer}~\cite{cf-gnnexplainer}
generates counterfactual explanations for GNNs via minimal edge deletions, 
but overlooks factual explanations.
(2) \underline{\cf}~\cite{tan2022learning} integrates factual and counterfactual reasoning for a specified test node into a single optimization problem, allowing the generation of GNN explanations that are both necessary and sufficient, but without robustness guarantees.
\eat{generates explanation subgraphs with counterfactual and factual constraints, and optimizes a relaxed objective.}

We are aware of two more recent methods 
~\cite{huang2023global} and 
~\cite{bajaj2021robust}. 
Both are optimized for counterfactual 
explanations. 
The former has
source codes relying on a pre-trained model; 
and the latter provides no available codes for node classification based explanation generation.
Neither discusses efficiency and scalability of explanation generation. 
We thus cannot provide a fair experimental 
comparison with them. 

\eat{
and the latter 
yet encountered challenges in reproducing their results. \gcf necessitates the use of pre-trained models from another study. }
\eat{
However, the publicly available code for generating this model is currently plagued with unresolved bugs. Similarly, \rcexp's codebase, while comprehensive with multiple options for varied experiments, is hampered by redundant files and codes, making it difficult to adapt to our dataset and methods. 
Despite these reproducibility issues, our analysis of  \rcexp 
suggest that the method is not 
addressing 
the robustness of their explanations, particularly in the context of $k$-edge disturbances. To our knowledge, our study is the first to address all three key criteria: robustness, factual and counterfactual,  aspects of explanations. Moreover, we delve into exploring the efficiency and scalability of explanation generation. 
}

\eat{\textcolor{red}{[AK: Add reasons why we do not show \gcf and RCExplainer. For RCExplainer, mention that the source code has not been provided by the authors. Also mention that RCExplainer, though considers robustness of explanations, does not provide any guarantee in terms of $k$-edge disturbance, unlike ours. State that none of the existing methods consider all three criteria, unlike ours. Add that none of these competitors consider efficiency and scalability of explanation generation, unlike ours.]}}

\eat{\gcf~\cite{huang2023global} 
finds a small set of representative
counterfactual graphs, similar to the input graph, but with opposite predictions. 
Since it only considers graph classification as the downstream task, we do not compare against \gcf.}

\stitle{Evaluation metrics}. We evaluate \rgexp and other explainers based on the following set of metrics. 

\sstab
(1) \underline{Normalized $\ged$}. 
Graph edit distance ($\ged$) measures the number of edits required to 
transform one graph to another. We use a normalized \ged (defined 
in Eq.~\ref{eq-ged}; where the graph size 
refers to the total number of nodes and edges) 
to quantify the structural similarity of generated 
\rcws, over their counterparts from disturbed graphs. 
The $\ged$ quantifies the ability that a $k$-\rcw $G_w'$ 
from a disturbed graph by a $k$-disturbance remains 
to be an ``invariant'' compared to its original counterparts 
$G_w$ before $k$-disturbances, which indicates their robustness: the smaller, the more ``robust''. 
\begin{small}
\begin{equation}
\label{eq-ged}
\text{normalized \ged}(G_w, G_w') = \frac{\ged(G_w, G_w')}{\max(|G_w|, |G'_w|)}
\end{equation}
\label{-2ex}
\end{small}
\eat{\textcolor{blue}{It measures the similarity between the embeddings of a test set and its $k$-RCW (Robust Counterfactual Witness) in \gnn}, \textcolor{red}{[AK: write more formally what is being done, what is being measured, and how?]}
indicating how well $k$-\rcw preserves the test set's features and structure. \textcolor{blue}{Consistency score remaining stable as $k$ increases suggest better $k$-\rcw robustness}\textcolor{red}{[AK: intuitively clarify why so?]}.
\eat{via graph embedding cosine similarity: We use a GNN\cite{kipf2016semi} to generate embeddings for both test graph and explanation subgraph. A higher similarity score indicates that the explanation subgraph maintains the meaningful information of test nodes.}}

\vspace{-1ex}
\sstab
(2) \underline{\fplus\cite{yuan2022explainability}.} It quantifies the deviations caused by removing the explanation subgraph from the input graph. Fidelity+ evaluates the counterfactual effectiveness. A higher Fidelity+ score indicates better distinction, hence the better. Below, $\mathbbm{1}$ is 1 if $M(v, G)=l$, and 0 otherwise. 

\begin{small}
\begin{equation*}
    Fidelity+ =\frac{1}{|V_T|}\mathop{\sum}_{v\in V_T}(\mathbbm{1}(M(v, G)=l)-\mathbbm{1}(M(v, G\setminus G_s)=l))
\end{equation*}
\end{small}

\sstab
(3) \underline{\fminus\cite{yuan2022explainability}.} 
It quantifies the difference between the results 
over $G$ and the generated \rcws, 
indicating a factual accuracy.\eat{measures how close the prediction results of the explanation subgraphs are to the original inputs. Fidelity- evaluates the factual effectiveness. A desirable Fidelity- score should be close to or even smaller than zero,} Ideal Fidelity- scores are low, even negative, indicating perfectly matched or even stronger predictions. 

\begin{small}
\begin{equation*}
    Fidelity- = \frac{1}{|V_T|} \mathop{\sum}_{v\in V_T}(\mathbbm{1}(M(v, G)=l)-\mathbbm{1}(M(v, G_s)=l))
\end{equation*}
\end{small}

\sstab
(4) \underline{Generation Time.} The total response 
time on generating explanations for all the test nodes 
$|V_T|$. To evaluate the impact of the factor $k$ on $k$-disturbances to generation time,  
we report the total response time it takes to 
``re-generate'' the explanations for each 
method. The learning-based methods \cf and \cfexplainer require 
retraining upon the change of graphs. 
We include the training and generation 
costs for both. 

\eat{\wu{as these two use different methods, add more to clarify what time 
are included -- training time? Give some proper 
details, not details, on how the time is calculated, 
for \cf and \cfexplainer.}}

\eat{\textcolor{red}{[AK: shall we report explanation size or sparsity, only in Table III, no plots. Change rows to columns for space saving.]}}

\eat{\textcolor{red}{[AK: (1) some idea should be given how test nodes in $V_T$ are selected (e.g., uniformly at random)?
(2) If a competitor identifies instance level explanation for each test node, how the overall explanation for $V_T$ is created (by graph union of instance-level explanations)?]}}

\vspace{-1ex}
\stitle{Environment}.
Our algorithms are implemented in Python 3.8.1 by PyTorch-Geometric framework~\cite{fey2019fast}. All experiments are conducted on a Linux system equipped with 2 CPUs, each possessing 16 cores 
and 256 GB RAM.

\subsection{Experiment results}

\vspace{-1ex}
\stitle{Exp-1: Effectiveness: quality of explanations (RQ1(a))}. 
Using \citeseer, 
we first evaluate the effectiveness of \rgexp, \cf, and \cfexplainer  
in terms of the quality of the explanations 
they generate.  The results over other 
datasets are consistent, hence omitted. 

\eetitle{Quality of Explanations}. 
Setting $k$ = $20$, and $|V_T|$ = $20$, we 
report the quality of the explanation structures 
of \rgexp, \cf, and \cfexplainer in Table~\ref{tab:effective}. 
The table shows that \rgexp outperforms both \cf and \cfexplainer for all three metrics. 
(1) \rgexp outperforms \cf and \cfexplainer by twice 
and $2.3$ times in terms of \ged. This demonstrates 
that it can generate explanations that are much 
more stable in structural similarity over 
the disturbance of underlying graphs. 
(2) \rgexp achieves best scores in \fplus 
and \fminus, outperforming \cf and \cfexplainer. 
These indicate that \rgexp can generate 
explanations as \rcws that stay consistent 
with the results of \gnn classification 
simultaneously as factual, and counterfactual explanations. 

We remark that the \rcws, by definition, should 
be theoretically achieving \fplus as $1.0$ and 
\fminus $0.0$, as counterfactual and factual 
witnesses, respectively. The reason 
that it does not achieve the theoretical 
best scores is due to the fact that not all the test nodes 
retain the same labels, leading to non-existence of 
non-trivial $k$-\rcws in some cases. 

We also compare the size of the explanations. \rgexp generate the smallest size, which is on average half of their \cf counterparts. \cf and \cfexplainer 
produce explanations as the union of subgraph 
structures at instance-level (for single test node), which may contain redundant nodes and 
edges that occur more than once.  The explanation 
subgraphs produced by \rgexp are more concise, due to which 
their generation targets for all test nodes in a greedy 
selection approach. 

\eat{
We also report the average size of 
the generated explanations in terms 
of the total number of nodes and 
edges in Table~\ref{tab:effective}. 
The total size for \cf and \cfexplainer is calculated by the graph union of instance-level explanations. The explanation of \rgexp has the smallest size which is twice smaller than \cf. This indicates that the generation of \rgexp targets small explanations for all test nodes, while the competitors focus on explaining around one single node. 
}

\eat{
\sstab
(1) In Table~\ref{tab:effective}, we present the effectiveness of each explainer that has been introduced in Section~\ref{sec:exp:settings}. Here, $k$ and $|V_T|$ have been set to $20$ and $20$, respectively. The table shows that \rgexp performs better than \cf and \cfexplainer across all the metrics listed.
The \ged score for \rgexp is approximately $2$ times lower than that of \cf, and $2.3$ times lower than \cfexplainer. This significant difference demonstrates that \rgexp's explanations have minor variations for up to $k$ disturbances in the input graph, 
compared to the competitors, suggesting that \rgexp has a more robust explanation capability.
Also, the Fidelity+ score of \rgexp is higher than competitors, making it more effective for counterfactual explanations as it has a greater impact on the model's prediction when its explanation subgraph is removed.
Moreover, \rgexp has a Fidelity- score of $0.05$, while \cf and \cfexplainer have higher scores 
indicating that \rgexp's explanations are more closely aligned with the actual predictions of the model compared to \cf and \cfexplainer.
}

\begin{table}[tb!]
\caption{Quality of explanations (\citeseer)} 
\vspace{-2.2ex}
\label{tab:effective}
\begin{small}
\begin{center}
\begin{tabular}{c|c|c|c|c}
\hline
& NormGED & Fidelity+ & Fidelity- & Size \\ 
\hline
\rgexp & 0.32 & 0.79 & 0.05 & 66 \\ \hline  
\cf & 0.68 & 0.47 & 0.06 & 132 \\ \hline
\cfexplainer & 0.72 & 0.65 & 0.13 & 78 \\ \hline 
\end{tabular}
\end{center}
\end{small}
\vspace{-5ex}
\end{table}

\eat{\begin{table}[tb!]
\caption{Quality of explanations (\citeseer)} 
\vspace{-2.2ex}
\label{tab:effective}
\begin{small}
\begin{center}
\begin{tabular}{c|c|c|c}
\hline
Metrics & \rgexp & \cf & \cfexplainer \\ 
\hline
NormGED & 0.32 & 0.68 & 0.72 \\ \hline  
Fidelity+ & 0.79 & 0.47 & 0.65 \\ \hline
Fidelity- & 0.05 & 0.06 & 0.13 \\ \hline 
Size & 66 & 132 & 78 \\ \hline
\hline
\end{tabular}
\end{center}
\end{small}
\vspace{-3ex}
\end{table}}

\stitle{Exp-2: Effectiveness: impact of factors (RQ1(b))}. 
We next investigate the impact of the size of 
disturbance $k$ and test size $|\V_T|$ on the effectiveness of \rgexp, \cf, and \cfexplainer. We report their performances in Fig~\ref{fig:effec}.

\eetitle{Varying $k$}. Fixing $|\V_T|=20$, we vary $k$ from $4$ to $20$, and report the results in Fig~\ref{fig:effec}(a), (c), and (e).  

\sstab
(1) Fig~\ref{fig:effec}(a) shows us that for all 
approaches, \ged increases for larger $k$. 
Increasing $k$ implies a more substantial disturbance, 
and accordingly larger variance to the underlying graph, 
hence larger differences in \rcws to retain the label 
of $V_T$. Due to the similar reason, \cf and \cfexplainer 
exhibit larger \ged. On the other hand, \rgexp 
consistently outperforms \cf and \cfexplainer. For example, 
it outperforms both when $k$ = $20$, with a \ged 
 that remains better than the best scores \cf and 
\cfexplainer can reach.

\eat{
\rgexp maintains a significantly smaller \ged score compared to other methods, indicating higher robustness. In contrast, \cf and \cfexplainer exhibit an increasing \ged score with increasing $k$. Increasing $k$ implies a substantial disturbance in the input graph. Notice that the worst case for \rgexp, i.e., at $k=20$, is even better than the best case for both competitors, i.e., at $k=4$. 
This suggests that \rgexp is more robust compared to competitors, as it is less susceptible to perturbations. 
}

\sstab
(2) Fig~\ref{fig:effec}(c) verifies that in general, 
all three methods achieve higher \fplus as $k$ becomes larger. 
Meanwhile, \rgexp maintains higher and more stable \fplus than \cf and \cfexplainer, 
for all cases. We found that with $|V_T|$ unchanged, 
larger $k$ allows \rgexp to explore more $k$-disturbances 
``outside'' of \rcws, and identify 
edges that are more likely to change the node labels. 
By incorporating such edges into 
\rcws, the generated explanations are also closer 
to the decision boundary. This verifies that \rgexp 
is able to respond better to larger structural variance of graphs and find counterfactual 
structures that better approximate the decision boundary 
of \gnns. \rgexp remains 
least sensitive due to its ability to find robust 
explanations under disturbances. 
\cfexplainer outperforms 
\cf due to its specialized learning 
goal towards counterfactual property. 
\eat{\comwu{do we have some explanation on why \fplus also increases as $k$ becomes larger for \cf and \cfexplainer? 
They are clearly using different approach. }}
\eat{
\mengying{A larger $k$ will broaden the analysis by incorporating more nodes into the explanation. This is particularly important for complex, multi-class datasets like \citeseer, as it will enable a more comprehensive understanding of \gnn behavior and lead to better explanations.}
}
 

\eat{
significantly impact the model's predictions, remaining effective even with larger test sets or more disturbances. On the other hand, both \cf and \cfexplainer show a more significant decrease. \cf achieves \tbf times lower than \rgexp in Fidelity+ when $k$ is increased to \tbf. This suggests lower counterfactual effectiveness amidst disruptions.
}

\begin{figure}[tb!]
\centering
\centerline{\includegraphics[width =0.46\textwidth]{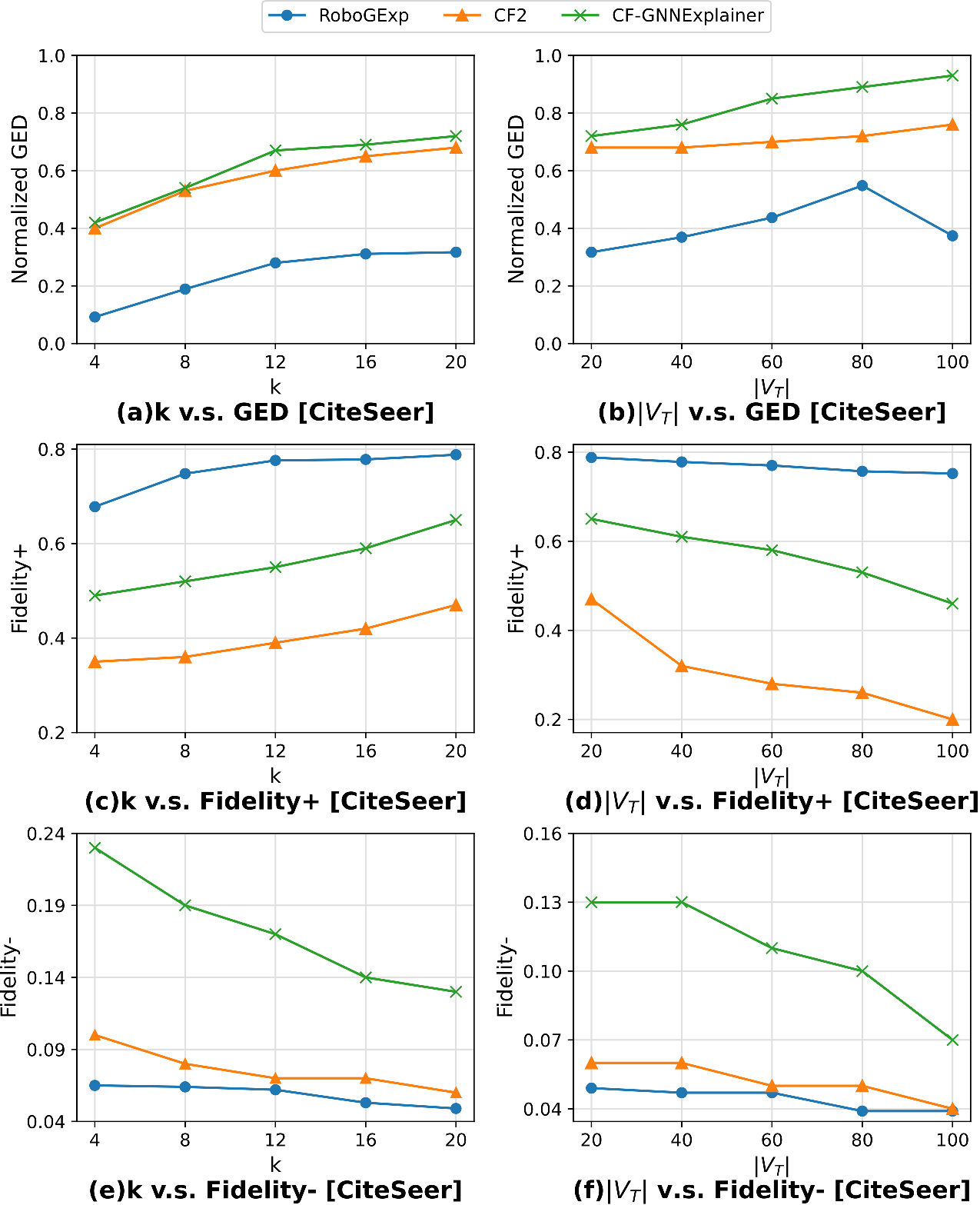}}
\vspace{-1.5ex}
\caption{Effectiveness with Impact Factors 
}
\label{fig:effec}
\vspace{-3ex}
\end{figure}

\sstab
(3) In Fig~\ref{fig:effec}(e), all methods 
have improved results with smaller \fminus 
as $k$ becomes larger. 
The erratic trend in \cf's \fminus score is indicative of its less consistent factual explanation generation process, particularly when the graph is subjected to disturbances. 
Comparing \cf and \cfexplainer, \cf 
is able to achieve \fminus 
scores that are closer to \rgexp, 
due to that it considers 
both factual and counterfactual explanations, 
while \cfexplainer is optimized for generating counterfactual explanations. 
Similar to \fplus, larger $k$ allows \rgexp to 
explore more disturbances and increases its chance to 
identify factual and counterfactual structures, 
resulting in a lower \fminus. 
\eat{\comwu{add why increased k help improve \fminus for \cf and \cfexplainer too.} }

\eetitle{Varying $|\V_T|$}. Fixing $k=20$, we vary $|\V_T|$ from $20$ to $100$, and present the results in Fig~\ref{fig:effec}(b), (d), and (f).

\sstab
(1) Fig~\ref{fig:effec}(b) depicts the following. 
 \rgexp remains to outperform \cf and \cfexplainer constantly in all cases. 
 Both \rgexp and \cf exhibit a relatively stable performance 
in \ged 
 as $|\V_T|$ varies from $20$ to $100$. 
 \cf is quite stable 
 due to its goal of seeking factual and 
 counterfactual explanations, which may have helped 
 it identify similar structures. \cfexplainer is quite sensitive to larger $|V_T|$, 
 due to larger test nodes introducing more 
 structurally different explanations. 

 \eat{
 This trend reflects \rgexp's capability to provide consistent explanations regardless of the test set size. In contrast, \cf and \cfexplainer report a much higher \ged compared to \rgexp, indicating their weaknesses in robustness.
 }

\sstab
(2) In Fig~\ref{fig:effec}(d), all methods 
have lower \fplus scores as $|V_T|$ becomes 
larger. As more test nodes are introduced, 
it becomes difficult to maintain 
a counterfactual explanation due to 
more diverse structures. 
\cf optimizes on both factual and counterfactual explanations, and  
demonstrates more sensitive behavior. \cfexplainer outperforms \cf 
in this case due to its optimization for only 
counterfactual explanations. On the other hand, 
\rgexp remains to outperform \cf and \cfexplainer, 
and is least sensitive to $|V_T|$. It is likely
that the goal to 
find stable explanation for all test nodes 
greatly mitigates the impact of the additional 
diversity introduced by larger $V_T$.

\eat{
the \fplus of \rgexp remain consistently high at around \tbf, 
indicating a strong counterfactual impact across different test set sizes. \cf's decrease is more pronounced, from \tbf to \tbf, thus the method's counterfactual reasoning becomes less effective as the size of the test set increases. The scores of \cfexplainer decrease by \tbf$\%$, which is closer to \cf than the stable \rgexp. 
}

\sstab
(3) In Fig~\ref{fig:effec}(f), all three methods find it difficult to 
maintain \fminus as $|V_T|$ is larger. 
\rgexp continues to achieve the best \fminus scores. 
The difference is that \cf in turn outperforms 
\cfexplainer and is able to achieve \fminus 
scores closer to \rgexp, due to its optimization 
on both factual and counterfactual explanations. In general, \rgexp remains the least sensitive to $V_T$, 
due to its prioritization strategy and enforcement 
of robustness verification.

\vspace{-1ex}
\stitle{Exp-3: Efficiency (RQ2)}. We report the efficiency of \rgexp, \cf, and \cfexplainer in Fig~\ref{fig:effic}. 
By default, we set $k = 20$, $|\V_T| = 20$, and test with \citeseer. 


\eetitle{Overall efficiency}. Fig~\ref{fig:effic}(a) shows the response time of \rgexp, \cf, and \cfexplainer for three real-world datasets: \house, \citeseer, and \ppi. \rgexp constantly 
outperforms \cf and \cfexplainer for all 
datasets. 
We found that 
the later two incur major overhead in the learning process to infer the explanations, which remains a main bottleneck and 
sensitive to large graphs with enriched features. 
Our methods explore bounded numbers of $k$-disturbances 
with efficient expansion and verification 
algorithms, 
and avoid unnecessary \gnn inferences. 
For example, \rgexp 
takes only 58.6\% of the time of \cfexplainer and 12.03\% of the time of \cf to find 
explanations that are factual, counterfactual, and 
robust. Moreover, upon the disturbance of 
graph $G$, \cf and \cfexplainer  
require retraining and ``re-generate'' explanations, while \rgexp can be applied 
to find robust explanations ``once-for-all'' 
for any variants of $G$ under $k$-disturbance. 

\eat{exhibits consistent performance across all datasets, with the lowest computation times. Notably, the computation time for \cf and \cfexplainer shows a sharp increase for the \ppi dataset, \rgexp's computation time on the \ppi dataset is only 58.6\% of the time taken for \cfexplainer and 12.03\% of the time for \cf, which could be attributed to the larger or more complex nature of this dataset. This suggests that our method performs best in dealing with more complex graph structures.
}

\eat{
\eetitle{Varying datasets}. Fig~\ref{fig:effic}(c) shows the computation time for different datasets: \house, \citeseer, and \ppi. It can be observed that \rgexp exhibits consistent performance across all datasets, with the lowest computation times. Notably, the computation time for \cf and \cfexplainer shows a sharp increase for the \ppi dataset, \rgexp's computation time on the \ppi dataset is only 58.6\% of the time taken for \cfexplainer and 12.03\% of the time for \cf, which could be attributed to the larger or more complex nature of this dataset. This suggests that our method performs best in dealing with more complex graph structures.
}

\begin{figure}[tb!]
\centering
\centerline{\includegraphics[width =0.46\textwidth]{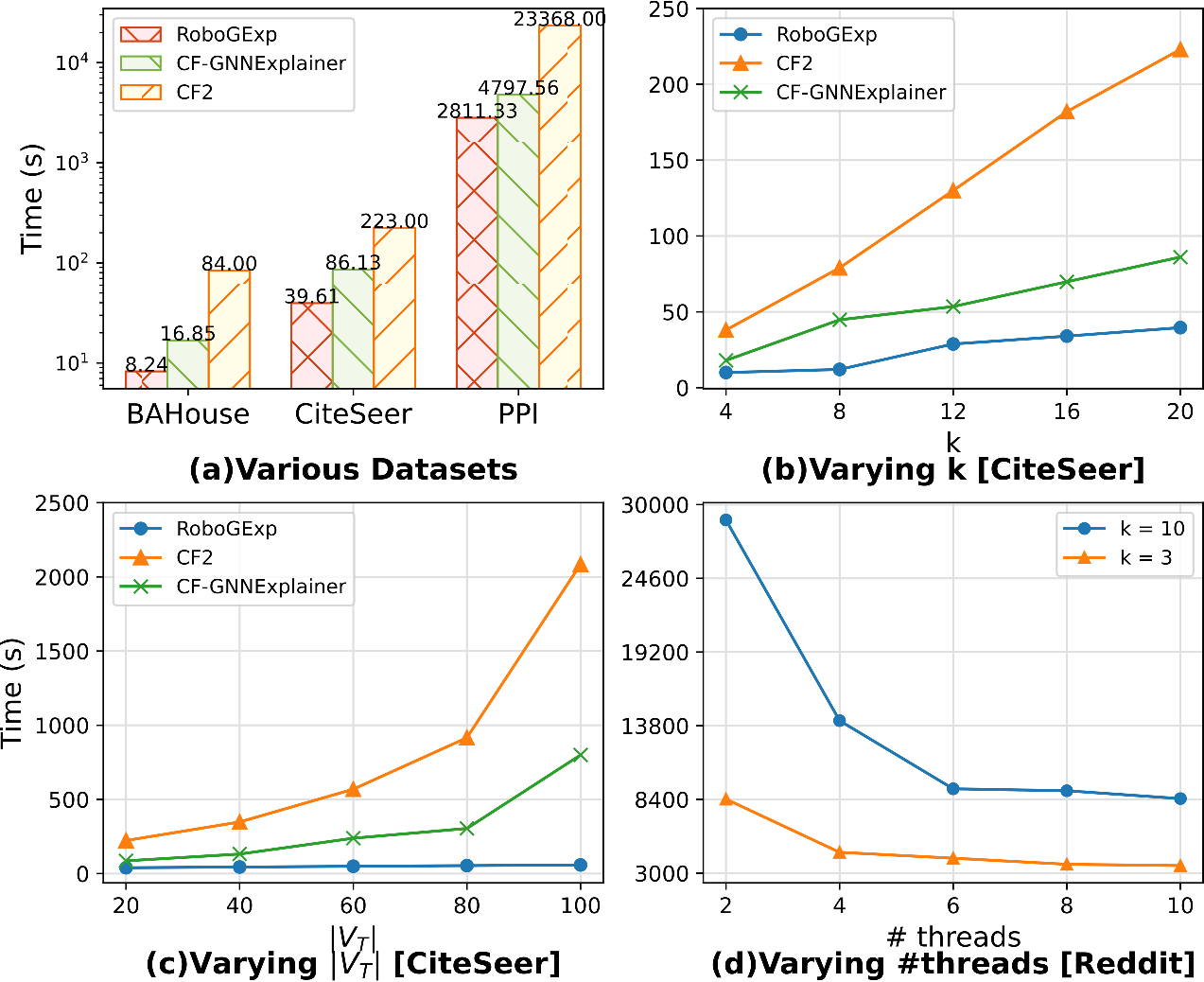}}
\vspace{-1.5ex}
\caption{Scalability and Efficiency 
}
\label{fig:effic}
\vspace{-3ex}
\end{figure}

\eetitle{Impact of $k$}. Using the same 
setting as in its effectiveness counterparts 
in Figs.~\ref{fig:effec} (a), (c), and (e), 
we report the response time of \rgexp, 
\cf, and \cfexplainer as $k$ varies from 
$4$ to $20$. All methods take more time as $k$ 
becomes larger. As expected, \rgexp 
takes more time to verify a larger number of 
$k$-disturbances, yet still benefits from 
its localized search in ``nearby'' 
area of the explanations. 
Both \cf and \cfexplainer incur increased overhead in learning; the difference is that 
\cf directly re-learns the matrix 
representation, 
and \cfexplainer re-learns to remove edges to sparsify the matrix, 
hence benefits better from removal-heavy 
disturbances.

\eat{
According to Fig~\ref{fig:effic}(b), the computation time for \rgexp minimally increases from 10s for $k$=4 to 39.6s for $k$=20, \textcolor{blue}{indicating a marginal increase rate of 1.85s per layer} \textcolor{red}{do we mean GNN layer? What increase in disturbances has to do with GNN Layers?}. In stark contrast, \cfexplainer and \cf result in a significant increase in computation time, \textcolor{blue}{with an increasing rate of 4.26s and 11.56s, respectively}. This is caused by the \textcolor{blue}{competitors requiring processing the entire adjacency matrix for each disturbance}, \textcolor{red}{[AK: unclear. explanation generation w.r.t. more disturbances in not an incremental process for our and competing methods. Given a graph -- disturbed or original, the methods produce an explanation.]} while ours efficiently computes explanations by expanding from $\V_T$ \textcolor{blue}{layer by layer}, minimizing the data processed in each round. These trends clearly illustrate \rgexp's superior efficiency in handling multiple disturbances, indicating that it is an efficient tool for generating robust \gnn explanations.
}

\eetitle{Varying $|\V_T|$}. Using the same 
setting as in its effectiveness counterparts 
in Figs.~\ref{fig:effec} (b), (d), and (f), 
we evaluate the impact of $|V_T|$ by varying 
it from 
$20$ to $100$. 
As shown in Figure \ref{fig:effic}(c), \rgexp scales well and is insensitive 
in response time, due to 
its prioritization strategy that 
favors the processing of test nodes which are 
unlikely to have labels changed given 
current explanations. In contrast, 
both \cf and \cfexplainer 
are sensitive, due to that 
both re-generate explanations each time 
$V_T$ changes, and incur
higher cost with 
more test nodes. 

\eat{
\rgexp's computation time demonstrates impressive efficiency, it only increases by 18.79s from $|\V_T|=20$ to $|\V_T|=100$. This is because \rgexp's computation cost benefits from generating explanations for the entire $\V_T$ in a unified process, thereby reducing redundant computations, so the increase in $|\V_T|$ does not affect computation time linearly. In contrast, \cfexplainer and \cf experience a steeper increase in computation time as $|\V_T|$ grows, compared to \rgexp. They generate explanations for each test node individually, which increases the overall computation time. This suggests that these competitors may not be efficient when scaling up to larger test sets.
}


\vspace{-1ex}
\stitle{Exp-4: Scalability (RQ3)}. We next 
evaluate the scalability of \prcwgen, in terms of the number of threads and 
$k$. Fig~\ref{fig:effic}(d) verifies the result over a large real-world dataset \reddit.  
\prcwgen scales well as the number of 
threads increases, with consistent impact 
due to larger $k$. The generation time is improved by 70.7\% as the number of threads varies 
from $2$ to $10$ when $k$=10. The result verifies that 
it is practical to generate explanations 
over large graphs, and the response time 
can be effectively better improved for larger-scale disturbances (i.e., when $k$ is larger).

\eat{which leverages \textcolor{blue}{intra-layer} parallelism \textcolor{red}{[AK: more commonly used term than 'intra-layer' parallelism? What does it mean? Layer means GNN layer?]} to optimize the \rgexp for efficient computations. }

\eat{
The results presented in Fig~\ref{fig:effic}(d) illustrate that our algorithm benefits significantly from parallel strategy on \reddit, which is a large-scale social network dataset with hundreds of millions of edges. As shown, the computation time reduces from approximately $3\times10^5$s with two threads to around $8\times10^3$s with ten threads at $k=10$, resulting in a reduction of more than 70.7\%, indicating that \prcwgen is capable of distributing the workload effectively \textcolor{blue}{within each layer} across multiple processors. It is worth noticing that the addition of threads from 2 to 4 leads to a decrease rate of 7351s and 1950s per thread in computation time for $k=3$ and $k=10$, respectively. The effect of different values of $k$ is due to the fact that with \prcwgen extending \textcolor{blue}{graph layers}, \textcolor{blue}{higher layers} may add more nodes, which amplifies the efficiency gains from \textcolor{blue}{intra-layer} parallelism. Also, the plateau is observed when increasing the thread count beyond 4, which is expected in parallel computing since adding threads increases coordination complexity. 
}


\eat{
\begin{figure}[tb!]
\centering
\centerline{\includegraphics[width =0.46\textwidth]{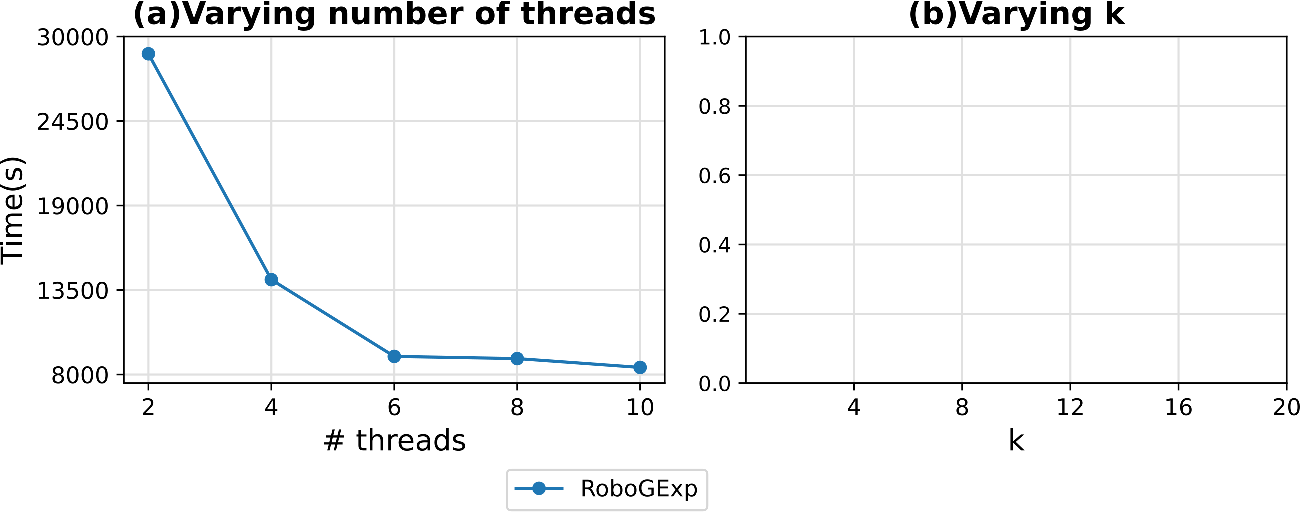}}
\vspace{-1.5ex}
\caption{Scalability and Parallization}
\label{fig:effic}
\end{figure}}

\vspace{-1ex}
\stitle{Exp-5: Case Analysis}. We next present two case studies, to showcase the application scenarios of \rgexp.

\eetitle{Deciphering invariant in drug structures.} Fig.~\ref{fig:case} depicts a chemical compound graph $G_3$ with a test node $v_3$ classified as \textit{``Mutagenic''}. On the top right is the \rcw $G_{w3}$ generated by \rgexp, which preserves the aldehyde structure in the blue area -- crucial for recognizing $v$ as a component from a toxic fraction. \cf generates a larger explanation $G^{\prime}_{w3}$. For two molecule variants $G^{1}_{3}$ and $G^{2}_{3}$ of $G_3$ obtained by removing the edge $e_7$ and edge $e_8$, respectively, 
\rgexp verifies that $G_{w3}$ remains 
to be a $1$-\rcw for the family of all three compounds, while 
\cf generates 
two different explanations $G^{\prime 1}_{w3}$ and $G^{\prime 2}_{w3}$. 
Observe that in both $G^{\prime 1}_{w3}$ and $G^{\prime 2}_{w3}$, their aldehyde structure contains one redundant connection to a hydrogen atom, which 
will change the label of $v_3$ to ``nonmutagen''. 
Indeed, while the structure O-C-H is easy to be recognized as a toxic fragment responsible 
for mutagenic property, and is preserved in $G_{w3}$, additional ``noisy'' substructure that contains O, C, and two H is included to \cf generated structures, 
making it hard to distinguish the node to be in a  
mutagenic fraction or not. 


\begin{figure}[tb!]
\centering
\centerline{\includegraphics[width =0.48\textwidth]{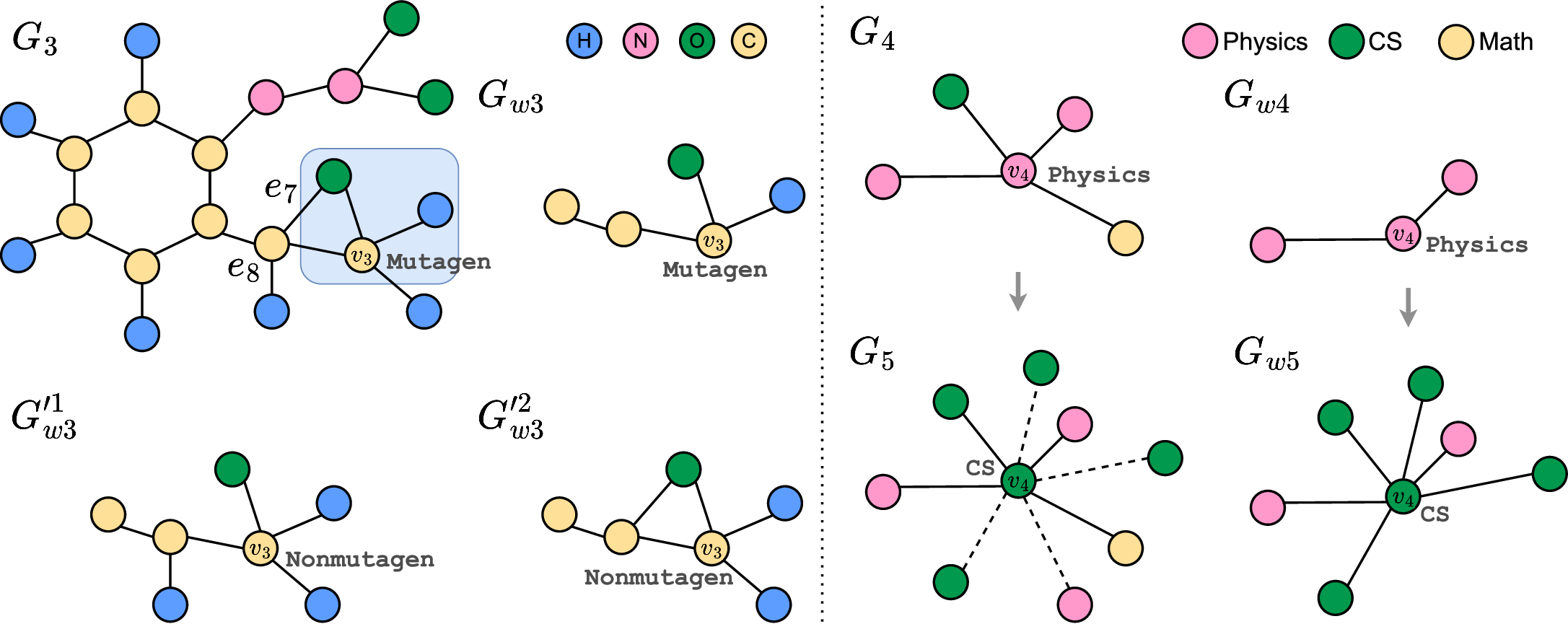}}
\caption{Case Study: Left: a small \rcw that capture an invariant 
structure for three variants 
of a drug structure; Right: 
an \rcw that explains topic 
change with matching new citations.}
\label{fig:case}
\vspace{-3ex}
\end{figure}

\eetitle{Explaining topic change with 
new references}. Our second case 
shows that \rgexp is able to respond quickly to 
the change of true labels with new explanation. 
Consider $G_4$ from \citeseer with a node (a paper) 
about Quantum Computing, which has the label `Physics', 
and has an explanation $G_{w4}$ found by \rgexp that identifies its relevant papers, all in physics area. 
Recent citations ``disturb'' $G_4$ to 
$G_5$ with new neighbors 
that cite the paper from computer sciences, via which 
the \gcn $\M$ updates its label 
to be ``Computer Science''. 
\rgexp responds by discovering a new explanation $G_{w5}$, with small changes that include 
a majority of major computer science citations
and less from physics. 
This verifies that \rgexp strikes a balance between factual and robustness as 
the true label changes.

\eat{\begin{figure}[tb!]
\addtolength{\subfigcapskip}{-0.04in}
\begin{center}
\subfigure[Case 1: Deciphering invariant structure in drug variants]{\label{fig:case:drug}
{\includegraphics[width =0.46\textwidth]{figures/case1.pdf}}}
\quad
\subfigure[Case 2: Indexing and querying complex graph data]{\label{fig:case:index}
{\includegraphics[width =0.46\textwidth]{figures/case2.pdf}}}
\end{center}
\vspace{-2ex}
\caption{Case Analysis \label{fig:case}}
\end{figure}
}

\section{Conclusion}
\label{sec:concl}

We proposed $k$-robust counterfactual witnesses ($k$-\rcw), an explanation 
structure for \gnn-based classification, that 
are factual, counterfactual, and robust to 
$k$-disturbances. We 
established the hardness and feasibility results, 
from tractable cases to co-\NP-hardness, 
for both verification and generation problems. 
We introduced feasible algorithms to 
tackle the verification and generation problems, and parallel algorithms 
to generate \rcws for large graphs. 
Our experimental study verified the 
efficiency of explanation generation, and the  
quality of the explanations, and their applications. 
A future topic is to enhance  
our solution to generate minimum 
explanations, 
and evaluate their application for 
other \gnns-based tasks.

\stitle{Acknowledgment}.
Dazhuo Qiu and Mengying Wang contributed equally. Qiu and Khan acknowledge support from the Novo Nordisk Foundation grant NNF22OC0072415. Wang and Wu are supported in part by NSF under CNS-1932574, ECCS-1933279, CNS-2028748 and OAC-2104007. 


\balance
\bibliographystyle{IEEEtran}
\bibliography{ref}


\end{document}